%% file: main.tex
\documentclass{article}

\usepackage{microtype}
\usepackage{graphicx}
\usepackage{booktabs} %

\usepackage{hyperref}

\usepackage[noend]{algpseudocode}
\usepackage[arxiv]{icml2024}

\usepackage{amsmath}
\usepackage{amssymb}
\usepackage{mathtools}
\usepackage{amsthm}

\usepackage[capitalize,noabbrev]{cleveref}

\usepackage[textsize=tiny]{todonotes}

\icmltitlerunning{Easy as ABCs: Unifying BQL and CFR}

\usepackage{xcolor} 
\usepackage{dsfont}
\usepackage{caption}
\usepackage{subcaption}
\usepackage{graphicx}

\newcommand*\Let[2]{\State #1 $\gets$ #2}
\algrenewcommand\alglinenumber[1]{
    {\sf\footnotesize\addfontfeatures{Colour=888888,Numbers=Monospaced}#1}}
\algrenewcommand\algorithmicrequire{\textbf{Precondition:}}
\algrenewcommand\algorithmicensure{\textbf{Postcondition:}}

\input{paper/macros}

\begin{document}

\twocolumn[
\icmltitle{Easy as ABCs: Unifying Boltzmann Q-Learning and Counterfactual Regret Minimization}

\icmlsetsymbol{equal}{*}

\begin{icmlauthorlist}
\icmlauthor{Luca D'Amico-Wong}{equal,harvard}
\icmlauthor{Hugh Zhang}{equal,harvard}
\icmlauthor{Marc Lanctot}{google}
\icmlauthor{David C. Parkes}{harvard,partlygoogle}
\end{icmlauthorlist}

\icmlaffiliation{harvard}{Harvard University}
\icmlaffiliation{google}{Google Deepmind}
\icmlaffiliation{partlygoogle}{Work partially done while at Google Deepmind}

\icmlcorrespondingauthor{Luca D'Amico-Wong}{ldamicowong@college.harvard.edu}
\icmlcorrespondingauthor{Hugh Zhang}{hughzhang@seas.harvard.edu}

\icmlkeywords{counterfactual regret minimization, q-learning}

\vskip 0.3in
]

\printAffiliationsAndNotice{\icmlEqualContribution} %

\input{paper/abstract}

\input{paper/introduction}

\input{paper/preliminaries}

\input{paper/cfr_and_bql}

\input{paper/stationarity}

\input{paper/algorithm}
\input{paper/experiments}
\input{paper/limitations}

\appendix
\bibliography{icml24}
\bibliographystyle{icml2024}

\newpage
\appendix
\onecolumn
\input{paper/appendix}

\end{document}

%% file: paper/macros.tex
\theoremstyle{plain}
\newtheorem{theorem}{Theorem}[section]

\newtheorem{lemma}[theorem]{Lemma}
\newtheorem{corollary}[theorem]{Corollary}
\theoremstyle{definition}
\newtheorem{definition}[theorem]{Definition}

\theoremstyle{remark}

\DeclarePairedDelimiter\autobracket{(}{)}
\DeclarePairedDelimiter\autosqbracket{[}{]}
\newcommand{\br}[1]{\autobracket*{#1}}
\newcommand{\sq}[1]{\autosqbracket*{#1}}
\newcommand{\norm}[1]{\left\lVert#1\right\rVert}

\newcommand{\E}[2]{\mathbb{E}_{#1}\sq{#2}}
\DeclareMathOperator*{\argmax}{arg\,max}

\definecolor{english}{rgb}{0.0, 0.5, 0.0} 
\definecolor{darkgreen}{rgb}{0.2, 0.6, 0.2}

\if 1
\usepackage{color}

\newcount\Comments  %
\newcommand{\kibitz}[2]{\ifnum\Comments=1{\color{#1}{#2}}\fi}

\newcount\Commentss  %
\newcommand{\kibitzs}[2]{\ifnum\Commentss=1{\color{#1}{#2}}\fi}
\definecolor{english}{rgb}{0.0, 0.5, 0.0}

\fi

\newcommand{\alg}[0]{ABCs}

\newcommand{\cs}[0]{child stationarity}

\renewcommand{\do}[0]{D^i_\pi(o)}

\newcommand{\cfr}[0]{MAX-CFR}

\newcommand{\bootcfr}[0]{Boot-CFR}

\DeclarePairedDelimiter\floor{\lfloor}{\rfloor}
\newcommand{\arxiv}[1]{}

\definecolor{light-gray}{gray}{0.95}

%% file: paper/abstract.tex
\begin{abstract}
We propose {\em \alg{}} (\textbf{A}daptive \textbf{B}ranching through \textbf{C}hild \textbf{s}tationarity), a best-of-both-worlds algorithm combining {\em Boltzmann Q-learning} (BQL), a classic reinforcement learning algorithm for single-agent domains, and {\em counterfactual regret minimization} (CFR), a central
 algorithm for learning in multi-agent domains. \alg{} adaptively chooses what fraction of the environment to explore each iteration by measuring the stationarity of the environment's reward and transition dynamics. In Markov decision processes, \alg{} converges to the optimal policy with 
at most an $O(A)$ factor slowdown compared to BQL, where $A$ is the number of actions in the environment. %
In two-player zero-sum games, \alg{} is guaranteed to converge to a Nash equilibrium (assuming access to a perfect oracle for detecting stationarity), while BQL has no such guarantees. Empirically, \alg{} demonstrates strong performance when benchmarked across environments drawn from the OpenSpiel game library and OpenAI Gym and exceeds all prior methods in environments which are neither fully stationary nor fully nonstationary.
\end{abstract}

%% file: paper/introduction.tex
\section{Introduction}
The ultimate dream of reinforcement learning (RL) is a general algorithm that can learn in any environment.
Nevertheless, present-day RL often requires assuming that the environment is stationary (i.e. that the transition and reward dynamics do not change over time).
When this assumption is violated, many RL methods, such as Boltzmann Q-Learning (BQL), fail to learn good policies or even converge at all, both in theory and in practice \cite{zinkevich2008regret,LaurentEtAl2011NonMarkovian,brown2020combining}.

Meanwhile, breakthroughs in no-regret learning, such as counterfactual regret minimization (CFR) \cite{zinkevich2008regret}, have led to tremendous progress in imperfect-information, multi-agent games like Poker and Diplomacy \cite{moravcik2017deepstack, brown2018superhuman, brown2019superhuman, metafundamentalairesearchdiplomacyteamfair2022humanlevel, bakhtin2022mastering}. Such algorithms are able to guarantee convergence to Nash equilibria in two-player zero-sum games, which are typically not Markov Decision Processes (MDPs). However, CFR has poor scaling properties due to the need to perform updates across the entire game tree at \emph{every} learning iteration, as opposed to only across a single trajectory like BQL.
As a result, CFR algorithms are typically substantially less efficient than their RL counterparts when used on stationary environments such as MDPs.
While Monte Carlo based methods such as {\em MCCFR} have been proposed to help alleviate this issue \cite{lanctot2009monte}, CFR algorithms often remain impractical in even toy reinforcement learning environments, such as Cartpole or the Arcade Learning Environment~\cite{sutton2018reinforcement,bellemare13arcade}.

We propose \alg{} (\textbf{A}daptive \textbf{B}ranching through \textbf{C}hild \textbf{s}tationarity), a best-of-both-worlds approach that combines the strengths of BQL and CFR to create a single, versatile algorithm capable of learning in both stationary and nonstationary environments.
The key insight behind \alg{} is to dynamically adapt the algorithm's learning strategy by measuring the stationarity of the environment. If an information state (a set of observationally equivalent states) is deemed to be approximately stationary, \alg{} performs a relatively cheap BQL-like update, only exploring a single trajectory. On the other hand, if it is nonstationary, \alg{} applies a more expensive CFR-like update, branching out all possible actions, to preserve convergence guarantees. This selective updating mechanism enables \alg{} to exploit the efficiency of BQL in stationary settings while maintaining the robustness of CFR in nonstationary environments, only exploring the game tree more thoroughly when necessary.

The precise notion of stationarity that \alg{} tests for is \emph{\cs{}}, a weaker notion than the Markovian stationarity typically assumed in MDPs. In an MDP, the reward and transition functions must remain stationary for the full environment, regardless of the policy chosen by the agent. Instead, \cs{} isolates the transition associated with a specific infostate $s$ and action $a$ and requires only that this transition remain stationary with respect to the policies encountered over the course of learning. The primary advantage of testing for child stationarity over full Markov stationarity is that this allows \alg{} to run efficiently in \emph{partially stationary} environments. 
In empirical experiments across Cartpole, Kuhn/Leduc poker, and a third novel environment combining the two to create a partially stationary environment, we find that \alg{} performs comparably to BQL on stationary environments while still guaranteeing convergence to equilibria in all environments. Moreover, in the partially stationary experiments, \alg{} outperforms both methods. Theoretically, under a slightly stronger assumption that we are given access to a perfect oracle for detecting child stationarity, we show that running a (fast) BQL update for transitions that satisfy \cs{} allows convergence to equilibria even if the remainder of the environment is nonstationary. 
\ifdefined\isarxiv
For reproducibility, all code is provided at \href{https://github.com/lucadwong/abcs}{https://github.com/lucadwong/abcs}.
\else
\fi

\subsection{Related Work}
\label{sec:related}
Discovering general algorithms capable of learning in varied environments is a common theme in RL and more generally in machine learning. For example, DQN~\cite{Mnih2015DQN} and its successors %
\cite{Wang16Dueling,BellemareDistDQN,Hessel18Rainbow}
are capable of obtaining human level performance on a test suite of games drawn from the Arcade Learning Environment, and algorithms such as AlphaZero \cite{silver2017mastering}  attain superhuman performance on perfect information games such as Go, Shogi, and 
Chess. However, such algorithms still fail on multi-agent imperfect information settings such as Poker~\cite{brown2020combining}.
More recently, several algorithms such as Player of Games (PoG) \cite{Schmid21PoG}, REBEL \cite{brown2020combining}, and MMD \cite{sokota2023unified} have been able to simultaneously achieve reasonable performance on both perfect and imperfect information games. However, unlike \alg, none of these algorithms adapt to the stationarity of their environment. As such, they can neither guarantee performance comparable to their reinforcement learning counterparts in MDPs nor efficiently allocate resources in environments which are partially stationary, as \alg{} is able to do.

%% file: paper/preliminaries.tex
\section{Preliminaries}
\subsection{Markov Decision Processes}
\label{sec:mdp}

Classical single-agent RL typically assumes that the learning environment is a \textit{Markov Decision Process} (MDP). An MDP is defined via a tuple $(S, A, T, R, \gamma)$. $S$ and $A$ are the set of  states and actions, respectively. The state transition function, $T(s^\prime \mid  s, a)$, defines the probability of transitioning to state $s^\prime$ after taking action $a$ at state $s$.
$R(r \mid s, a, s')$ is the reward function, which specifies the reward for taking action $a$ at state $s$, given that we transition to state $s'$. Importantly, in an MDP, both the transition and the reward functions are assumed to be \emph{stationary/fixed}.
$\gamma \in [0, 1]$ is the discount factor determining the importance of future rewards relative to immediate rewards.
In an {\em episodic MDP}, the agent interacts with the environment across each of separate
{\em  episodes} via a
{\em  policy}, which is a function $\pi: S \rightarrow \Delta A$, mapping
 states to probabilities over valid actions.
In a \textit{Partially Observable Markov Decision Process} (POMDP), the tuple contains an additional element $\mathcal{H}$. In a POMDP, the agent only observes {\em information states} $S$ while the true hidden states $\mathcal{H}$ of the world remain unknown. The transition and reward dynamics of the environment  follow an MDP over the  hidden states $\mathcal{H}$, but the agent's policy $\pi: S \to \Delta A$ must depend on the observed infostates rather than hidden states. However, since POMDPs still only contain a single agent and fix the transition probabilities over time, they are still stationary environments, and there exist standard methods for reducing POMDPs into MDPs \cite{Shani13survey}.

\subsection{Finite Extensive-Form Games}
While MDPs and POMDPs model a large class of single-agent environments, the stationarity assumption can easily fail 
 in multi-agent environments. Specifically, from a single agent's perspective, the transition and reward dynamics 
are no longer fixed as they now depend on the other agents, who are also learning. 
To capture these multi-agent environments, we cast our environment as a {\em finite, imperfect information, extensive-form game with perfect recall}.

In defining this, we have a set of $M$ players (or agents), and an additional \emph{chance player}, denoted by $c$, whose moves capture any potential stochasticity in the game.
There is a finite set $\mathcal{H}$ of possible %
\emph{histories/hidden states}, corresponding to a finite sequence of states encountered and actions taken by \emph{all} players in the game. Let $h \sqsubseteq h'$ denote that $h$ is a prefix of $h'$, as sequences. The set of \emph{terminal histories}, $Z \subseteq \mathcal{H}$, is the set of histories where the game terminates. For each history $h \in \mathcal{H}$, we let $r_i(h)$ denote the \emph{reward} obtained by player $i$ upon reaching history $h$.
The total \emph{utility} in a given episode for player $i$ after reaching terminal history $z \in Z$ is  $u_i(z) = \sum_{h_t \sqsubseteq z} \gamma^{t} r_i(h_t)$, where $h_t$ denotes the $t+1$st longest prefix of $z$ (i.e., $t$ indexes a step within an episode).
 Let $u_i(h, h')$ denote the total utility for player $i$ when reaching history $h'$, starting from history $h$. 
Many of the guarantees in this paper focus on two-player zero-sum games, which have the additional property that the sum of the utility functions at any $z \in Z$ for the two players is guaranteed to be $0$.
 
To model imperfect information, we define the set of \emph{infostates} $S$ as a partition of the possible histories $\mathcal{H}$. Let $S(h)$ denote the \emph{information states} (abbreviated infostates) corresponding to history $h$. There is a \emph{player function} $P: S \to \{1, \dots, M\} \cup \{c\}$ that maps infostates to the player whose turn it is to move. Each infostate $s \in S$ is composed of states that are observationally equivalent to the player $P(s)$ whose turn it is to play. As is standard, we assume   \emph{perfect recall} in that players never forget information obtained through prior states and actions \citep{zinkevich2008regret, brown2018superhuman, moravcik2017deepstack, celli2020noregret}. 

Let $\mathcal{A}(s)$ denote the set of possible actions available to player $P(s)$ at infostate $s$. We denote the policy of player $i$ as $\pi_i: S\rightarrow\Delta \mathcal{A}$ and let $\pi = \{\pi_1, \cdots \pi_N\}$ denote the \emph{joint policy}. 
Let $\mathcal{H}_\pi(s)$ denote the distribution of true hidden states corresponding to infostate $s$, assuming all players play according to joint policy $\pi$.

In order to bridge the gap between POMDPs and general extensive-form games, it will  be helpful to recast our multi-agent environment into $|M|$ separate single-agent environments. From the perspective of any single agent $i$, the policy of the other agents $\pi_{-i}$ uniquely defines a POMDP which agent $i$ interacts with alone. To construct this POMDP, after player $i$ takes an action, we simulate the play of the other agents according to $\pi_{-i}$ and return the state/rewards corresponding to the point at which it is next $i$'s turn to play. Within these single-agent POMDPs, defined by the joint policy $\pi$, we use $T_{\pi}\br{h'\mid h, a}$ to denote the probability\footnote{Note that these transition probabilities are only defined for the player whose turn it is to play at history $h$} that $h'$ is the next hidden state after player $i$ plays action $a$ under true hidden state $h$.
 We can similarly define the transition probability to infostate $s'$ from hidden state $h$ as $T_{\pi}\br{s'\mid h, a} = \sum_{h' \in \mathcal{H}_\pi(s')} \mathds{1}\br{S\br{h'} = s'}T_{\pi}\br{h'\mid h, a}$, and the transition probabilities  from infostate to infostate
 as $T_{\pi}\br{s'\mid s, a} = \E{h \sim \mathcal{H}_\pi(s)}{T_{\pi}\br{s'\mid h, a}}$. Analogously, let $R_{\pi}\br{r\mid s, a, s'}$ denote the probability of obtaining immediate reward $r$ after taking action $a$ at infostate $s$ under joint policy $\pi$ and transitioning to state $s'$. For the above transitions to always be well defined, we require that all policies be fully mixed, in that $\pi(a \mid s) > 0$ for all valid actions $a$ and all infostates $s$. It is important to note that the $h', s'$ above do not refer to the next history/infostate seen in the full game but rather the next  in the single-agent POMDP of the current player---i.e., the agent who actually takes action $a$.

%% file: paper/cfr_and_bql.tex
\section{Unifying Q-Learning and Counterfactual Regret Minimization}

The following notation will be useful for describing both BQL and CFR.
Let $\eta^{\pi}(h)$ be the probability of reaching history $h$ if all players play policy $\pi$ from the start of the game. Let $\eta^{\pi}_{-i}(h), \eta^{\pi}_i(h)$ denote the contribution of all players except player $i$ and the contribution of player $i$ to this probability respectively. Analogously, let $\eta^{\pi}(s) = \sum_h \mathds{1}\br{S(h) = h} \eta^{\pi}(h)$ be the probability of reaching infostate $s$ and $\eta^{\pi}(h, h')$ be the probability of reaching history $h'$ starting from history $h$. Finally, let $Z_s$ denote the set of all $(h, z)$ such that $S(h) = s$ and $z \in Z$ is a possible terminal history when starting from $h$.

\subsection{Boltzmann Q-Learning}

Boltzmann Q-learning (BQL) is a variant of the standard Q-learning algorithm, where Boltzmann exploration is used as the exploration policy. Formally, we  define the value $V$ of a state and the $Q$-value of an action at that state as:
\begin{align*}
    V_i^{\pi}(s) &= \sum_{(h, z) \in Z_s} \frac{\eta^{\pi}(h)}{\eta^{\pi}(s)} \eta^{\pi}(h, z) u_i(h, z) \\
    Q^{\pi}(s, a) &= \underset{{\substack{s' \sim T_{\pi}(s'|s, a) \\ r \sim R_{\pi}(r|s, a, s')}}}{\mathbb{E}} \left[r + \gamma V_{i}^{\pi}(s')\right]
\end{align*}
where the $i$ in the definition of $Q^\pi(s, a)$ refers to the player who plays at state $s$.

We adapt BQL to the multi-agent setting as follows. 
At each iteration, BQL first freezes the policies of each agent. It then samples a single trajectory for each agent. These trajectories take the form of $(s, a, r, s')$ tuples, recording the reward $r$ and new infostate $s'$ observed after taking action $a$ at infostate $s$. Q-values are updated for each $(s, a)$ pair in the trajectory using a temporal difference update $Q(s, a) \leftarrow (1-\alpha) Q(s, a) + \alpha (r + \gamma \max_{a'} Q(s', a'))$, where $\alpha$ denotes the learning rate.
These average Q-values are stored, and at the end of each iteration $n$, a new joint policy $
\pi^{n+1}(s, a) = \frac{\exp\left(Q(s, a)/\tau\right)}{\sum_{a' \in A} \exp\left(Q(s, a')/\tau\right)}$ is computed, where $\tau$ is the temperature parameter.
While BQL converges to the optimal policy in MDPs and in perfect information games (given a suitable temperature schedule) \cite{cesa-bianchi2017boltzmann, singh2000convergence}, it fails
 to converge to a Nash equilibrium in imperfect-information multi-agent games such as Poker \cite{brown2020combining}. %

\subsection{Counterfactual Regret Minimization}

Counterfactual regret minimization (CFR) is a popular method for solving imperfect-information extensive form games and is based on the notion of minimizing \emph{counterfactual regrets}. 
The counterfactual value for a given infostate is given by $v_i^{\pi}(s) = \sum_{(h, z) \in Z_{s}} \eta_{-i}^{\pi}(h) \eta^{\pi}(h, z) u_i(h, z)$. The \emph{counterfactual regret} for not taking action $a$ at infostate $s$ is then given by $r^{\pi}(s, a) = v_i^{\pi_{(s \to a)}}(s) - v_i^{\pi}(s)$, where $i = P(s)$  and $\pi_{(s \to a)}$ is identical to policy $\pi$ except action $a$ is always taken at infostate $s$. Notice that because $\eta^\pi_{i}(s) = \eta^\pi_{i}(h)$ (due to the assumption of perfect recall), the counterfactual values are simply the standard RL value functions normalized by the opponents' contribution to the probability of reaching $s$; that is, $v_{i}^\pi(s) = \frac{V_i^\pi(s)}{\eta^\pi_{-i}(s)}$\footnote{Intuitively, the counterfactual regret upper bounds how much an agent could possibly ``regret'' not playing action $a$ at infostate $s$ if she modified her policy to maximize the probability that state $s$ is encountered. As such, the normalization constant excludes player $i$'s probability contribution to reaching $s$ under $\pi$.}. This was first pointed out by \cite{srinivasan2018actor}. 
As with multi-agent BQL, at each iteration, we freeze all agents' policies before traversing the game tree and computing the counterfactual regrets for each $(s, a)$ pair. A new current policy $\pi^{n+1}$ is then computed using a regret minimizer, typically regret matching \cite{hart2000simple} or Hedge \cite{arora2012online} on the cumulative regrets $R^n(s, a) = \sum_n r^{\pi^n}(s, a)$, and the process continues.
 A learning procedure has \emph{vanishing regret} if $\lim_{N\rightarrow\infty} \frac{1}{N} R^N(s, a) = 0$ for all $s, a$. In two player, zero-sum games, \cite{zinkevich2008regret} prove that CFR has vanishing regret, and by extension, that the average policy converges to a Nash equilibrium of the game. In general-sum games, computing a Nash equilibrium is computationally hard \cite{daskalakis2009complexity}, but CFR guarantees convergence to a weaker solution concept known as a coarse-correlated equilibrium \cite{morrill2021hindsight}.

\subsection{External Sampling MCCFR with Hedge}
External sampling MCCFR (ES-MCCFR) is a popular variant of CFR which uses sampling to get unbiased estimates of the counterfactual regrets without exploring the full game tree~\cite{lanctot2009monte}. 
For each player $i$, ES-MCCFR samples a single action for every infostate which players other than $i$ act on, but explores all actions for player $i$.
Let $W$ denote the set of terminal histories reached during this sampling process. For any infostate $s$, let $h(s)$ denote the corresponding history that was reached during exploration (there can only be one due to perfect recall). For each visited infostate $s$, the sampled counterfactual regret, $\tilde{r}_i(s, a) = \tilde{v}^{\pi_{(s\to a)}}_i(s, a)  - \tilde{v}^{\pi}_i(s, a) = \sum_{z \in W} u_i(h, z)(\eta_i^\pi(h(s'), z) - \eta_i^\pi(h(s), z))$, is added to the cumulative regret, where $s'$ is the infostate reached after taking action $a$ at infostate $s$. After this process has been completed for each player, a new joint policy is computed using Hedge on the new cumulative regrets, and the procedure continues.
Under Hedge, the joint policy at episode $n+1$
is given by $\pi^{n+1}(s, a) = \frac{e^{R^n(s, a)/\tau_n}}{\sum_{a' \in A(s)} e^{R^n(s, a')/\tau_n}}$, where $\tau_n$ is a temperature hyperparameter. As Hedge is invariant to any additive constants in the exponent, we can replace the sampled counterfactual regrets with the sampled counterfactual values, i.e., $\tilde{v}^{\pi}_i(s, a) = \sum_{z \in W} u_i(z) \cdot \eta_i^\pi(h(s'), z)$, and similarly with the cumulative regrets.

\subsection{BQL vs. CFR}
\label{sec:compare}

The most important difference between CFR and BQL is in regard to  which infostates are updated at each iteration of learning. By default, BQL only performs updates along  a single trajectory of encountered infostates (``trajectory based").
 In a game of depth $D$, this means that there will be at most $O(D)$ updates per iteration. By contrast, CFR and most of its variants are typically\footnote{In theory, variants such as OS-MCCFR can learn with $O(D)$ updates per iteration, but the variance of such methods is often too high for practical usage. Other approaches, such as ESCHER~\cite{mcaleer2022escher}, have attempted to address this issue by avoiding the use of importance sampling.}
 not trajectory-based learning algorithms because many infostates that are updated are not on the path of play. In particular, CFR \cite{zinkevich2008regret}  requires performing counterfactual value updates at  every
 infostate, which can require up to $O(A^D)$ updates per iteration. 
While MCCFR reduces this by performing Monte-Carlo updates, variants such as ES-MCCFR still require making an exponential number of updates per iteration.

Despite the differences between BQL and CFR, the two algorithms are intimately linked. \citep[Lemma 1]{brown2019deep} first pointed out that the sampled counterfactual values $\tilde{v}_i^{\pi}(s, a)$ from ES-MCCFR have an alternative interpretation as Monte-Carlo based estimates of the Q-value $Q^\pi(s, a)$. 
Thus, ES-MCCFR with Hedge has a gradient update that ends up being very similar to BQL, with two exceptions. Firstly, BQL estimates its Q-values by bootstrapping whereas ES-MCCFR uses a Monte-Carlo update and thus updates based on the \emph{current} policy of all players.
Secondly, BQL chooses a policy based on a softmax over the \emph{average} Q-values, while CFR(Hedge) uses a softmax over \emph{cumulative} Q-values. As not every action is tested in BQL, the average Q-values may be based on a different number of samples for each action, making it unclear how to convert average Q-values into cumulative Q-values. %

%% file: paper/stationarity.tex
 \section{Learning our \alg{}}
 \label{sec:algorithm}

 As described in Section~\ref{sec:compare}, while the updates that CFR(Hedge) and BQL perform at each infoset are quite similar,  BQL learns over trajectories, whereas CFR expands most (if not all) of the game tree at \emph{every infostate}. This exploration is necessary in nonstationary environments where the Q-values may change over time but  wasteful in stationary environments.
As a concrete  example, consider the following game. Half of the game is comprised of a standard game of poker, and half of the game is a ``dummy'' game identical to poker except that the payoff is always $0$.
In the ``poker'' subtree, BQL will fail to converge because this subtree is not an MDP. However, on the ``dummy'' subtree of the game, CFR will continue to perform full updates and explore the entire subtree, even though simply sampling trajectories on this ``dummy'' subtree would be sufficient for convergence.

{\em What if we could formulate an algorithm that could itself discover when it needed to expand a child node,
 and when it could simply sample a trajectory?} The hope is that this algorithm
 could closely match the performance of RL algorithms in stationary environments while  retaining performance guarantees in nonstationary ones. 
Moreover, one can easily imagine general environments containing elements of both – playing a game of Poker followed by a game of Atari, for example –
 where such an algorithm could yield better performance than either CFR or BQL.
 
\subsection{Child Stationarity}
\label{sec:childstationarity}

Motivated by the goal of adaptive exploration, we define the important concept of \emph{child stationarity}, which is a relaxation of the standard Markov stationarity requirement. Let $\sigma: \Delta \Pi$ denote a distribution over joint policies.
\begin{definition}[Child stationarity]
 An infostate $s$ and action $a$ satisfy {\em \cs{} with respect to distributions over joint policies $\sigma$ and $\sigma'$} if $\E{\pi \in \sigma}{T_{\pi}(h^\prime \mid s, a)} = \E{\pi \in \sigma'}{T_{\pi}(h^\prime \mid s, a)}$ and $\E{\pi \sim \sigma}{R_{\pi}(r \mid s, a, s^\prime)} = \E{\pi \sim \sigma'}{R_{\pi}(r \mid s, a, s^\prime)}$, where $s'$ denotes the infostate observed after playing $a$ at $s$.
\end{definition}

An infostate action pair $(s, a)$ satisfies child stationarity for policy distributions $\sigma$ and $\sigma'$ if the local reward and state transition functions, from the perspective of the current player,
remain stationary despite a change in policy distribution,
 and even if other portions of the environment, possibly
 including ancestors of $s$, are nonstationary.
 Importantly, the transition function must remain constant with respect to the  true, hidden states
 and not merely the observed infostates.\footnote{We give a more detailed discussion  of the similarities and differences between \cs{} and Markovian stationarity in Appendix~\ref{markov_vs_cs}.}
 By definition, an MDP satisfies child stationarity for all infostates $s$ and actions $a$ with respect to all possible distributions over joint policies $\sigma, \sigma^\prime$. In contrast, in imperfect-information games such as Poker, the distribution over the hidden state at your next turn, after playing a given action $a$ at infostate $s$, will generally depend on the policies of the other players, violating child stationarity.

The utility of  child stationarity can be seen as follows. While MDPs and perfect information games can be decomposed into subgames that are solved independently from each other \cite{tesauro1995temporal, silver2016mastering, sutton2018reinforcement}, the same is not true for their imperfect-information counterparts \cite{brown2020combining}. Finding transitions that satisfy \cs{} allows for natural decompositions of the environment that can  be solved independently from the rest of the environment. As such, \cs{} forms a core reason for why \alg{} is able to unify BQL and CFR into a single algorithm.

Define $G^\sigma(s, a)$ as a game whose initial hidden state is drawn from $\E{\pi \in \sigma}{T_{\pi}(h^\prime \mid s, a)}$ and then played identically to $G$, the original game, from that point onward. $G^\sigma(s, a)$ is analogous to the concept of Public Belief States introduced by ReBeL \citep[Section 4]{brown2020combining}.
\begin{theorem}
\label{thm:equivalentgame}
If $s, a$ satisfy child stationarity with respect to $\sigma, \sigma'$, then $G^\sigma(s, a) = G^{\sigma'}(s, a)$. Furthermore, if $\pi^*$ is a Nash equilibrium of $G^\sigma(s, a)$, then $\pi^*$ is also a Nash equilibrium of $G^{\sigma'}(s, a)$.
\end{theorem}

\begin{algorithm}[t] 
    \caption{Returns True if $s, a$ is not child stationary.}
    \label{alg:detector}
    \begin{algorithmic}
    \Require{infostate $s$, action $a$, child hidden states $h_{1:N}$ (states following $(s,a)$), total episodes $N$, significance level $\alpha_s>0$.}
    \Statex
    \Function{DETECT}{$s, a, r_{1:N}, h_{1:N}, \alpha_s$}
        \Let{$X_1$}{\{$r_n, h_n \mid n \in 1 \cdots \floor{N/2}\}$}
        \Let{$X_2$}{\{$r_n, h_n \mid n \in \floor{N/2} + 1 \cdots N\}$}
        \Let{\text{pval}}{ChiSquared($X_1, X_2)$}
        \\
        \Return{$pval < \alpha_s$}
    \EndFunction
    \end{algorithmic}
    \end{algorithm}

\begin{proof}
By the definition of \cs{}, $\E{\pi \in \sigma}{T_{\pi}(h^\prime \mid s, a)} = \E{\pi \in \sigma'}{T_{\pi}(h^\prime \mid s, a)}$. Thus, $G^\sigma(s, a)$ and $G^\sigma(s, a)$ are informationally equivalent, as all players have the same belief distribution over hidden states at the beginning of the game and will thus have identical equilibria.
\end{proof}

Child stationarity forms the crux of the adaptive exploration policy of \alg{}, allowing us to employ trajectory-style updates at stationary $(s, a)$ pairs without affecting the convergence guarantees of CFR in two-player zero-sum games.

\subsection{Detecting Child Stationarity}

\label{sec:measuringchildstationarity}
How can we measure child stationarity in practice? 
Notice that the definition 
only requires that the transition and reward dynamics stay fixed with respect  to two particular distributions of joint policies, not all of them. 
Let $\pi_1, \pi_2, \cdots \pi_N$ be the policies used by our learning procedure for each iteration of learning from $1 \cdots N$. 
Define $\sigma^N_A$ as the uniform distribution over $\{\pi_1 \cdots \pi_{\floor*{N/2}}\}$ and $\sigma^N_B$ as the uniform distribution over $\{\pi_{\floor*{N/2} + 1} \cdots \pi_N\}$.
At iteration $N$, Algorithm~\ref{alg:detector} directly tests \cs{} with respect to distributions over joint policies $\sigma_A^N, \sigma_B^N$ as follows: 
for each infostate $s$ and action $a$, it keeps track of the rewards and true hidden states that follow playing action $a$ at $s$. 
 Then, it runs a Chi-Squared Goodness-Of-Fit test \cite{pearson1900criterion} to determine whether the distribution of true hidden states has changed between the first and second half of training. We adopt \cs{} as the null hypothesis. 
 While Algorithm~\ref{alg:detector} is not a perfect detector, Theorem~\ref{thm:nofalsenegatives} shows that in the limit,  Algorithm~\ref{alg:detector} will asymptotically never claim that a given $s, a, \sigma_A, \sigma_B$ satisfies \cs{} when it does not, assuming $(s, a)$ is queried infinitely often. However, the Chi Squared test retains a fixed rate of false positives given by the significance level $\alpha_s$, so Algorithm~\ref{alg:mainalg} will incorrectly reject the null hypothesis of \cs{} with some positive probability.
 \begin{theorem}
 \label{thm:nofalsenegatives}
 As $|X_1|, |X_2| \rightarrow \infty$, the probability that Algorithm~\ref{alg:detector} falsely claims $s, a$ satisfies \cs{} when it does not goes to $0$. 
 \end{theorem}
 \begin{proof}
 This statement is directly implied by the fact that the Chi-Squared Goodness-of-Fit test guarantees that the Type II error vanishes in the limit \cite{pearson1900criterion}.
 \end{proof}

%% file: paper/algorithm.tex
\subsection{The \alg{} Algorithm}

We are now ready to learn our \alg{}. At a high level, \alg{} chooses between a BQL update and a CFR update for each $Q(s, a)$ value based on whether or not it satisfies child stationarity with respect to the joint policies encountered over the course of training.  %
One can view \alg{} from two equivalent perspectives. From one perspective, \alg{} runs a variant of ES-MCCFR, except that at stationary infostates, where it is unnecessary to run a full CFR update, it defaults to a cheaper BQL update. Alternatively, one can view \alg{} as defaulting to a variant of BQL until it realizes that an infostate is too nonstationary for BQL to converge. At that point, it backs off to more expensive CFR style updates. We provide the formal description of \alg{} in Algorithm~\ref{alg:mainalg}. In the multi-agent setting, \alg{} freezes all players' policies at the beginning of each learning iteration and then runs independently for each player, exactly as multi-agent BQL and ES-MCCFR does. 

\algdef{SE}[SUBALG]{Indent}{EndIndent}{}{\algorithmicend\ }%
\algtext*{Indent}
\algtext*{EndIndent}
      \begin{algorithm}[t]
       \caption{\alg{}}
       \label{alg:mainalg}
    \begin{algorithmic}
        \Require{$Q$, CNT initialized to $0$, discount factor $\gamma$, total episodes $N$. Initial observation and hidden states $s_0, h_0$. $H = \varnothing$. Significance value $p$, exploration $\epsilon$ }
        \Statex
        \For{$n \gets 1 \textrm{ to } N$}
                \State \alg{}$\br{h_0, s_0, Q, \gamma, n, a_0, H, p, \epsilon}$ \Comment{$a_0$ is null}
        \EndFor
        \Statex
        \Function{GetChild}{$h, s, a, \pi, Q, \gamma$}
            \State Sample $h^\prime, s^\prime \sim T^{\pi}\br{h^\prime \mid h, a}, S(h^\prime)$
            \State Sample $r \sim R^{\pi}\br{r \mid s, a, s^\prime}$
            \Let{$a^*$}{$\argmax_{a} Q\br{s^\prime, a}$}
            \Let{$\nabla_{s, a}$}{$\br{r + \gamma Q\br{s', a^*}} - Q(s, a)$} 
            \State \Return{$h', s', a^*, \nabla_{s, a}$}
        \EndFunction
        \Statex
        \Function{\alg{}}{$h, s, Q, \gamma, n, a_{opt}, H, p, \epsilon$}
            \If{h is terminal}
                \Return{0}
            \EndIf
                \Let{$\pi^n$}{$\texttt{softmax}\br{Q\br{s, *}, \tau=1/(\tau_n \cdot \text{CNT}(s))}$}
            \Let{$\text{CNT}(s)$}{$\text{CNT}(s) + 1$}
            \State Sample $a_{traj}$ from $\pi^n(s)$ with $\epsilon$-exploration

            \For{$a \in A(s)$}
                \Let{$h^\prime, s^\prime,  a^*, \nabla_{s, a}$}{GetChild$\br{h, s, a, \pi^n, Q, \gamma}$}
                \State{Append $r, h^\prime$ to $H(s, a)$}
                \If{DETECT$\br{s, a, H(s, a), p}$ OR $a = a_{traj}$}
                    \Let{$\nabla_{s^\prime, a^*}$}{$\text{\alg{}}\br{h^\prime, s^\prime, Q, \gamma, t, a^*, H, p, \epsilon}$}
                    \If{s is not stationary}
                        \Let{$\nabla_{s, a}$}{$\nabla_{s, a} + \text{CNT}(s^\prime) \cdot \nabla_{s^\prime, a^*}$}
                    \EndIf
                \EndIf
                \Let{$Q(s, a)$}{$Q(s, a) + \frac{1}{ \text{CNT}(s)}\nabla_{s, a}$}
            \EndFor
            \Let{$a_{opt}$}{$\argmax Q(s, *)$}
            \State \Return{$\nabla_{s, a_{opt}}$}
        \EndFunction
    \end{algorithmic}
    \end{algorithm}

\paragraph{Where \alg{} Runs Updates} As described in Section~\ref{sec:compare}, one major difference between BQL and ES-MCCFR is the number of infostates each algorithm updates per learning iteration. While BQL updates only infostates encountered along a single trajectory, \alg{} adaptively adjusts the proportion of the game tree it expands by measuring child stationarity at each $(s, a)$ pair. Concretely, at every infostate encountered (regardless of its stationarity), \alg{} chooses an action $a_{traj}$ from its policy (with $\epsilon$-exploration) and recursively runs itself on the resulting successor state $s^\prime$.
For all other actions $a' \neq a$, it tests whether the tuple $(s, a')$ satisfies child stationarity. Unless the null is rejected, it prunes all further updates on that subtree for this iteration. Otherwise, when $(s, a')$ is detected as nonstationary,
 it also recursively runs on the successor state sampled from $T(s' \mid s, a')$. As such, if all actions $a$ fail the child stationarity test, then \alg{} will expand all child nodes, just like ES-MCCFR, but if the environment is an MDP and all $s, a$ satisfy child stationarity (according to the detector), then \alg{} will only expand one child at each state and thus approximate BQL style trajectory updates. With sufficient $\epsilon$-exploration (and because at least one action $a_{traj}$ will be expanded for each infostate $s$), \alg{} satisfies the requirement of Theorem~\ref{thm:nofalsenegatives} that every infostate $s$ will be visited infinitely often in the limit.

\paragraph{Q-value updates} In terms of the actual Q-value update that \alg{} performs, recall that in Section~\ref{sec:compare}, we describe two differences between BQL and ES-MCCFR. Firstly, ES-MCCFR updates its Q-values based on the \emph{current} policy while BQL performs updates based on the \emph{average} policy. To determine which update to make, \alg{} measures the level of \cs{} for a given $(s, a)$ pair, as it does when choosing \emph{which} infostates to update. If $(s, a)$ satisfies child stationarity, then \alg{} does a regular BQL update. If  not, \alg{} adds a nonstationarity correction factor $\text{CNT}(s^\prime) \cdot \nabla_{s^\prime, a^*}$ to the Q-value update,  correcting for the difference between the instantaneous Q-values sampled via Monte-Carlo sampling and the average bootstrapped Q-values. This correction exactly recovers the update function of MAX-CFR, a variant of ES-MCCFR described in more detail in Appendix~\ref{sec:maxcfr}. %

Secondly, ES-MCCFR chooses its policy as a function of \emph{cumulative} Q-values, whereas BQL uses \emph{average} Q-values. To address this, we force \alg{} to test each action upon reaching an infostate, and thus update $Q(s, a)$ for all $a$. This allows us to set the temperature to recover a scaled multiple of the cumulative Q-values. Specifically, by setting a temperature of $\frac{1}{\tau_n \cdot CNT(s)}$, where $CNT(s)$ counts the number of visits to $s$, we recover a softmax over the cumulative Q-values scaled by $\tau_n$, just like with CFR(Hedge). BQL is well known to converge to the optimal policy in an MDP as long as we are greedy in the limit and visit every state, action pair infinitely often, which is guaranteed with an appropriate choice of $\tau_n$ and $\epsilon$-schedule \cite{singh2000convergence, cesa-bianchi2017boltzmann}.

\paragraph{Convergence Guarantees} We formally prove the following theorems describing the performance of \alg{} in MDPs and two-player zero-sum games. Theorem~\ref{thm:constantfactor} proves that \alg{} converges no slower than an $O(A)$ factor compared to BQL in an MDP  using only the detector described in Section~\ref{sec:measuringchildstationarity}. Additionally, Theorem~\ref{thm:convergencetone} proves that a minor variant of \alg{} finds a Nash equilibrium in a two-player zero-sum game (assuming access to a perfect oracle). The full proofs are included in the supplementary material.

\begin{theorem}
\label{thm:constantfactor}
Given an appropriate choice of significance levels $p$ (to control the rate of false positives), \alg{} (as given in Algorithm~\ref{alg:mainalg}) asymptotically converges to the optimal policy with only a worst-case $O(A)$ factor slowdown compared to BQL in an MDP, where $A$ is the maximum number of actions available at any infoset in the game.
\end{theorem}
 
To guarantee convergence to a Nash equilibrium in two-player zero-sum games, we require two modifications to Algorithm \ref{alg:mainalg}. Firstly, we assume access to a perfect oracle for detecting child stationarity. Secondly, to ensure that stationary and nonstationary updates do not interfere with each other, we keep track of separate Q-values, choosing which to update and use for exploration according to whether the current $(s, a)$ pair satisfies child stationarity.

\begin{theorem}
\label{thm:convergencetone}
Assume that Algorithm~\ref{alg:mainalg} tracks separate $Q^{stat}(s, a)$ and $Q^{nonstat}(s, a)$ values for stationary and nonstationary updates and has access to a perfect oracle for detecting \cs{}. Then, the average policy $\lim_{N\rightarrow \infty} \frac{1}{N} \sum_{n=1}^N \pi_n$ converges to a Nash equilibrium in a two-player zero-sum game with high probability.
\end{theorem}

In our practical implementation of \alg{}, we make several simplifications compared to the assumptions required in theory. Since assuming access to a perfect oracle is unrealistic in practice, we use the detector described in Algorithm~\ref{alg:detector}, which guarantees asymptotically vanishing Type II error, but nevertheless retains some probability of a false positive in the limit. As described in Algorithm~\ref{alg:mainalg}, we only track a single Q-value for both stationary and nonstationary updates. Finally, we use a constant significance value for all infostates. Our practical implementation of \alg, despite not fully satisfying the  assumptions of Theorems~\ref{thm:constantfactor} and~\ref{thm:convergencetone}, nonetheless achieves comparable performance with both BQL and CFR in our experiments.

We also allow for different temperature schedules for stationary and nonstationary infostates, and we only perform the stationarity check with probability $0.05$ at each iteration, significantly reducing the time spent on these checks. We find that these modifications simplify our implementation without greatly affecting performance.

%% file: paper/experiments.tex
\section{Experimental Results}

We evaluate \alg{} across several single and multi-agent settings drawn from both the OpenSpiel game library \cite{lanctot2019openspiel} and OpenAI Gym \cite{brockman2016openai}, benchmarking it against MAX-CFR, BQL \cite{sutton2018reinforcement}, OS-MCCFR \cite{lanctot2009monte}, and ES-MCCFR \cite{lanctot2009monte} (when computationally feasible). All experiments with \alg{} were run on a single 2020 Macbook Air M1 with 8GB of RAM. We use a single set of hyperparameters and run across three random seeds, with 95\% error bars included. Hyperparameters and detailed descriptions of the environments are given in Appendix~\ref{sec:hyperparameters}.
Auxiliary experiments with TicTacToe are also included in Appendix~\ref{sec:add_experiments}. All plots describe regret/exploitability, so a lower line indicates superior performance.

\begin{figure*}[t]
    \centering
    \begin{subfigure}[t]{0.29\textwidth}
        \centering
        \includegraphics[width=\linewidth]{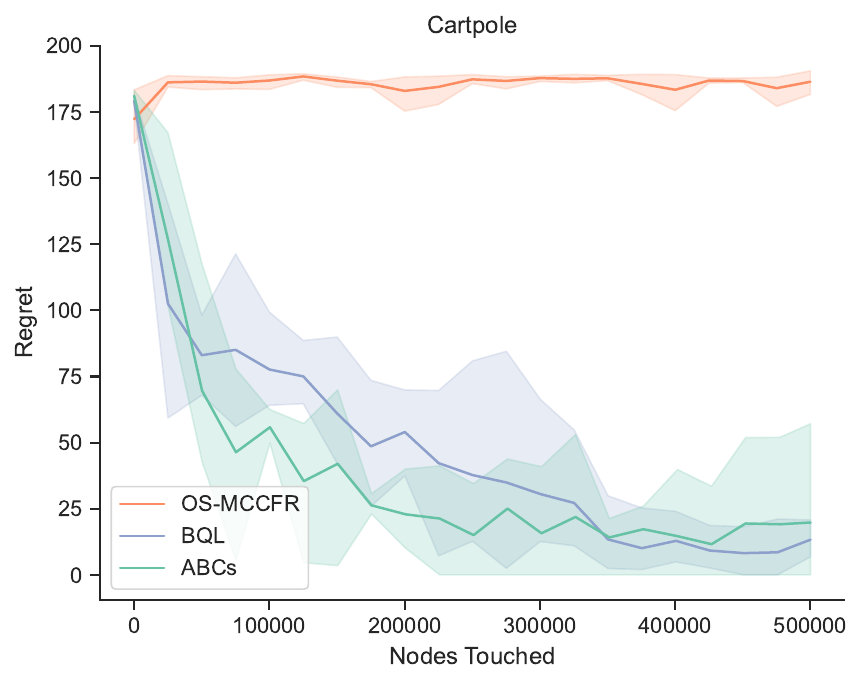} 
        \caption{Cartpole} \label{fig:cartpole}
    \end{subfigure}
    \setcounter{subfigure}{2}%
    \begin{subfigure}[t]{0.29\textwidth}
        \centering
        \includegraphics[width=\linewidth]{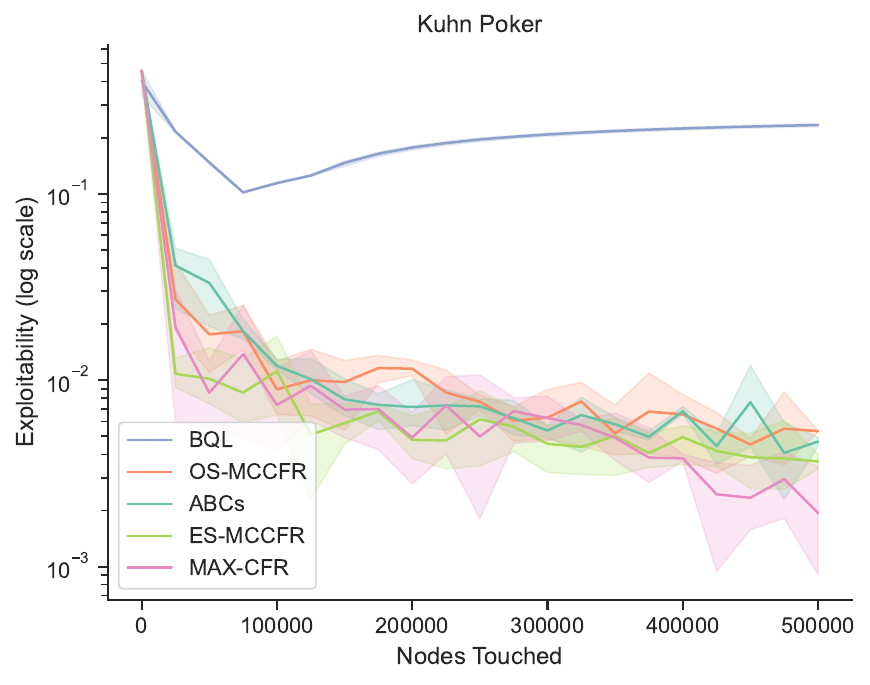} 
        \caption{Kuhn Poker} \label{fig:kuhn}
    \end{subfigure}
    \setcounter{subfigure}{4}%
    \begin{subfigure}[t]{0.29\textwidth}
        \centering
        \includegraphics[width=\linewidth]{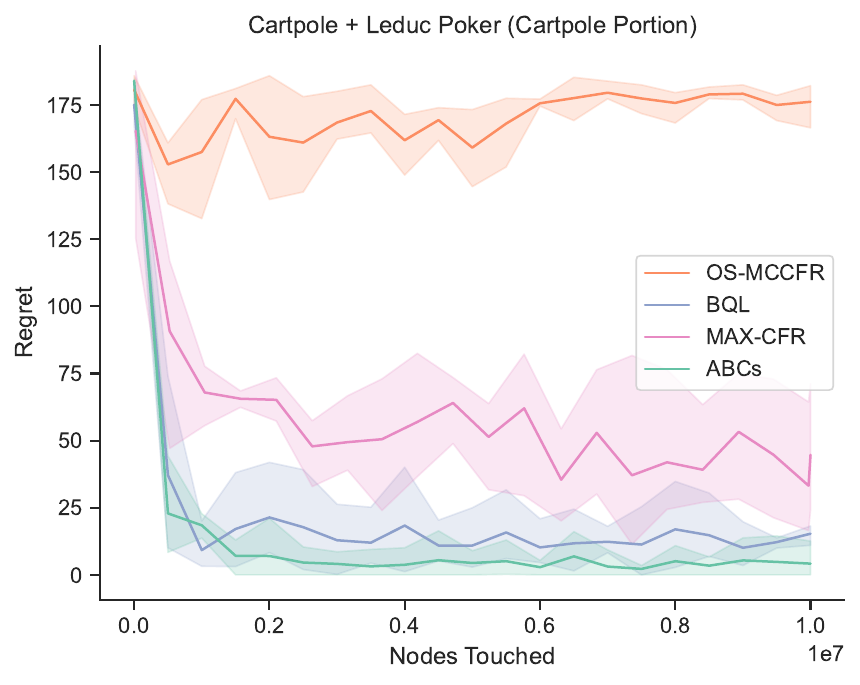} 
        \caption{Cartpole (Stacked Environment)} \label{fig:stacked_cartpole}
    \end{subfigure}

    \setcounter{subfigure}{1}%
    \begin{subfigure}[t]{0.29\textwidth}
        \centering
        \includegraphics[width=\linewidth]{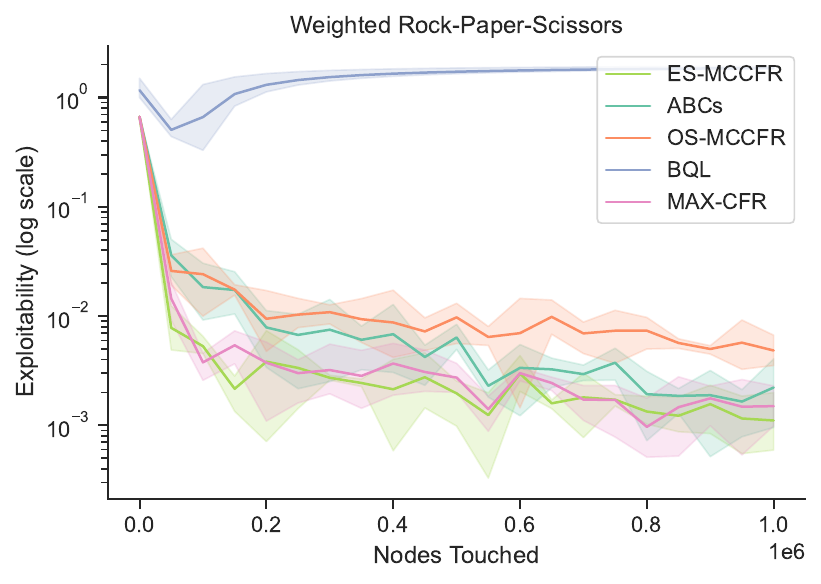} 
        \caption{Weighted RPS} \label{fig:weighted_rps}
    \end{subfigure}
    \setcounter{subfigure}{3}%
    \begin{subfigure}[t]{0.29\textwidth}
        \centering
        \includegraphics[width=\linewidth]{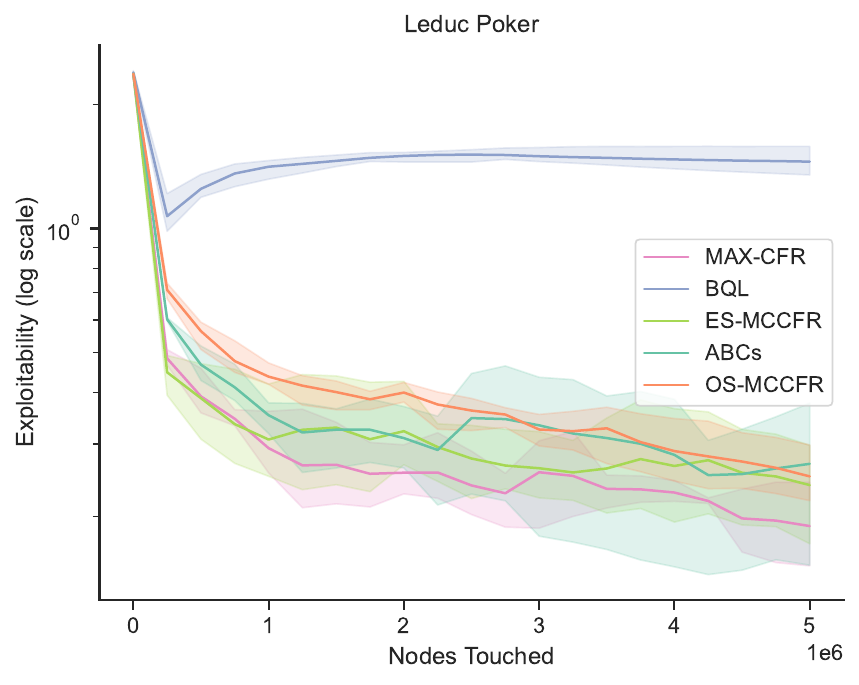} 
        \caption{Leduc Poker} \label{fig:leduc}
    \end{subfigure}
    \setcounter{subfigure}{5}%
    \begin{subfigure}[t]{0.29\textwidth}
        \centering
        \includegraphics[width=\linewidth]{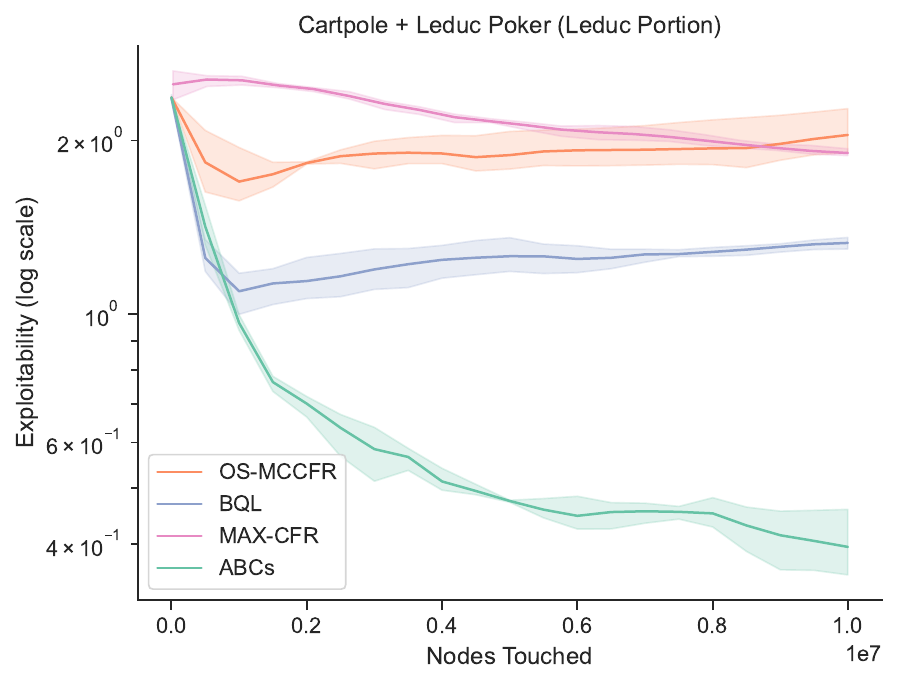} 
        \caption{Leduc (Stacked Environment)} \label{fig:stacked_leduc}
    \end{subfigure}

\caption{\alg{} matches the performance of BQL on stationary environments like Cartpole (a) and the performance of CFR methods across several non-stationary environments including weighted rock-paper-scissors (b), Kuhn poker (c), and Leduc poker (d). In a partially stationary environment with elements of both stationarity and nonstationarity, \alg{} outperforms both BQL and CFR, being the only algorithm capable of efficiently solving both the Cartpole (e) and Leduc poker (f) portion of the stacked environment.}
\end{figure*}

\paragraph{Cartpole}
We evaluate \alg{} on OpenAI Gym's Cartpole environment, a classic benchmark in reinforcement learning\footnote{We modify Cartpole slightly to ensure that it satisfies the Markovian property. See Appendix~\ref{sec:hyperparameters} for details.}.
Figure~\ref{fig:cartpole} shows that \alg{} performs comparably to BQL and substantially better than OS-MCCFR. ES-MCCFR is infeasible to run due to the large state space of Cartpole.

\paragraph{Weighted Rock Paper Scissors}
Weighted rock-paper-scissors is a classic nonstationary environment in which rock-paper-scissors is played, but there is twice as much reward if a player wins with Rock. 
Figure \ref{fig:weighted_rps} demonstrates that while BQL fails to learn the optimal policy, ABCs performs comparably to the standard MCCFR benchmarks.
\paragraph{Kuhn and Leduc Poker}
Kuhn and Leduc Poker are simplified versions of the full game of poker which operate as classic benchmarks for equilibrium finding algorithms. 
As illustrated in Figures~\ref{fig:kuhn} and \ref{fig:leduc}, BQL fails to converge on either Kuhn or Leduc Poker as they are not Markovian environments, and we show similar results for a wide variety of hyperparameters in Appendix~\ref{sec:add_experiments}. However, both \alg{} and MCCFR converge to a Nash equilibrium of the game, with \alg{} closely matching the performance of the standard MCCFR benchmarks.
\paragraph{A Partially Nonstationary Environment}
We also test \alg{}'s ability to adapt to a partially nonstationary environment, where its strengths are most properly utilized. Specifically, we consider a stacked environment, where each round consists of a game of Cartpole followed by
 a game of Leduc poker against another player. 
Figures \ref{fig:stacked_cartpole} and \ref{fig:stacked_leduc} show that BQL does not achieve the maximum score in this combined game,  
failing to converge on Leduc poker. Although
CFR based methods are theoretically capable of learning both games, MAXCFR wastes substantial time conducting unnecessary exploration of the game tree in the larger Cartpole setting, failing to converge on Leduc. OS-MCCFR similarly fails to learn either environment, though in this case, the issue preventing convergence on Leduc relates more to variance induced by importance sampling on such a deep game tree rather than unnecessary exploration. As such, \alg{} outperforms all MCCFR variants and BQL by better allocating resources between the stationary and nonstationary portions of the environment.

%% file: paper/limitations.tex
\section{Conclusion}
We introduce \alg{}, a unified algorithm combining the best parts of both Boltzmann Q-Learning and Counterfactual Regret Minimization. Theoretically, we show that \alg{} simultaneously matches the performance of BQL up to an $O(A)$ factor in MDPs while also guaranteeing convergence to Nash equilibria in two-player zero-sum games (under the assumption of access to a perfect oracle for stationarity). To our knowledge, this is the first such algorithm of its kind. Empirically, we show that these guarantees do not just lie in the realm of theory. In experiments, \alg{} performs comparably to BQL on stationary environments and CFR on nonstationary ones, and is able to exceed the performance of both algorithms in partially nonstationary domains. Our empirical experiments do not require any assumptions regarding access to a perfect oracle and use the detector described in Algorithm~\ref{alg:detector}.

We give a number of limitations of \alg{}, which we leave to future work. While Theorem~\ref{thm:convergencetone} guarantees that \alg{} converges to a Nash equilibrium in two-player zero-sum games, it both requires access to a perfect oracle and does not provide a bound on the rate of convergence.
Additionally, we run experiments for \alg{} only on relatively small environments due to computational reasons.
For future work, we hope to scale \alg{} to larger settings such as the Arcade Learning Environment \cite{bellemare13arcade} and larger games such as Go or Poker using function approximation and methods such as DQN \cite{mnih2013playing} or DeepCFR \cite{brown2019deep}.

\ifdefined\isarxiv
\else
    \section{Broader Impacts}
    This paper presents work whose goal is to advance the field of Machine Learning. There are many potential societal consequences of our work, none which we feel must be specifically highlighted here.
\fi

%% file: paper/appendix.tex
\section{Nash Equilibria and Exploitability}

We formally describe a Nash equilibrium and the concept of exploitability.
Let $\pi: S \to \Delta (\mathcal{A)}$ denote the \emph{joint policy} of all players in the game. We assume that $\pi$ has full support: for all actions $a \in \mathcal{A}(s)$, $\pi(a \mid s) > 0$. Let $\Sigma$ denote the space of all possible joint policies, with $\Sigma_i$ defined accordingly. Let $u_i(\pi)$ denote the expected utility to player $i$ under joint policy $\pi$. A joint policy $\pi$ is a \emph{Nash equilibrium} of a game if for all players $i$, $\pi_i \in \argmax_{\pi'_i \in \Sigma_i} u_i(\pi'_i, \pi_{-i})$. Similarly, policy $\pi$ is an \emph{$\epsilon$-Nash equilibrium} if for all players $i$, $u_i(\pi) \geq \max_{\pi'_i \in \Sigma_i} u_i(\pi'_i, \pi_{-i}) - \epsilon$.

In a two-player zero-sum game, playing a Nash equilibrium  guarantees you the minimax value; in other words, it maximizes the utility you receive given that your opponent always perfectly responds to your strategy. \emph{Exploitability} is often used to measure the distance of a policy $\pi$ from a Nash equilibrium, bounding  the worst possible losses from playing a given policy. The total exploitability of a policy $\pi$ is given by $\sum_{i \in M} \max_{\pi'_i \in \Sigma_i} u_i(\pi'_i, \pi_{-i}) - u_i(\pi)$.
In a two-player zero-sum game, the exploitability of a policy $\pi$ is given by $\max_{\pi'_1 \in \Sigma_1} u_1(\pi'_1, \pi_2) + \max_{\pi'_2 \in \Sigma_2} u_2(\pi_1, \pi'_2).$

\section{Markov Stationarity vs. Child Stationarity}
 \label{markov_vs_cs}

In this section, we  highlight the relationship between Markov stationarity and \cs{}. 
An MDP satisfies \cs{} at all infostates and actions $s, a$ and with respect to all possible distributions over joint policies $\sigma, \sigma'$ (see  Section~\ref{sec:childstationarity}).

 Here we highlight an environment that does not satisfy Markov stationarity but  satisfies \cs{}.
Consider an environment with a single infostate and two possible hidden states. Furthermore, let $h_n = h_{n \mod 2}$. Such an environment is not stationary in the Markov sense, as the hidden states change with each episode. However, defining $\sigma_A, \sigma_B$ as the empirical distribution of joint policies across the first and second half of timesteps as described in Section~\ref{sec:measuringchildstationarity}, such an environment would (asymptotically) satisfy \cs{} as the distribution of hidden states for the first and second half would approach $\br{\frac{1}{2}, \frac{1}{2}}$ in the limit. 
As shown in Section~\ref{sec:abcsconvergence}, \cs{}, despite being weaker than Markov stationarity still allows BQL to converge to the optimal policy.

\input{paper/maxcfr}

\section{Convergence of ABCs}
\arxiv{If there isn't a major mistake, the below proof doesn't require a perfect oracle. It just requires one sided error (which we have). It is still written under the assumption of a perfect oracle though.}

\label{sec:abcsconvergence}

{\bf Theorem 4.4.}
{\em Assume that Algorithm~\ref{alg:mainalg} tracks separate $Q(s, a)$ values for stationary and nonstationary updates and has access to a perfect oracle for detecting \cs{}. Then, the average policy $\lim_{N\rightarrow \infty} \frac{1}{N} \sum_{n=1}^N \pi_n$ in Algorithm~\ref{alg:mainalg} converges to a Nash equilibrium in a two-player zero-sum game with high probability.}
\begin{proof}
Call our current game $G$. We will prove convergence to a Nash equilibria in $G$ by showing  the average policy in Algorithm~\ref{alg:mainalg} has vanishing local regret. %
This is sufficient to show  the algorithm as a whole has no-regret and thus converges to a Nash equilibrium \citep[Theorem 2]{zinkevich2008regret}.
 
  We first define a perfect oracle. At every episode $N$, a  perfect oracle has full access to the specification of  game $G$ and  the  joint policy, $\pi_1 \cdots \pi_N$, of all players in each episode $1$ through $N$.
As per Section~\ref{sec:measuringchildstationarity}, define $\sigma^N_A$ as the uniform distribution over $\{\pi_1 \cdots \pi_{\floor*{N/2}}\}$ and $\sigma^N_B$ as the uniform distribution over $\{\pi_{\floor*{N/2} + 1} \cdots \pi_N\}$.
Define the following two distributions:
\begin{align*}
    D_1^N(h) & = \E{\pi \in \sigma^N_A}{T_{\pi}(h^\prime \mid s, a)} \\
    D_2^N(h) & = \E{\pi \in \sigma^N_B}{T_{\pi}(h^\prime \mid s, a)}.
\end{align*} 

Define the  total variation distance  between these two distributions \cite{billingsley2012probability},  as
\[
\Delta_N = \sum_{h \in \mathcal{H}} \norm{D^N_1(h) - D^N_2(h)}.
\] 

A perfect detector  will return that $s, a$ is child stationary  at iteration $N$ if and only if $\Delta_N = 0$.
 
We assume separate Q-value tables for the stationary and nonstationary updates, so that the updates do not interfere with each other. To determine the policy
at each iteration (or episode) $n$, we chose $\pi^n(a) \propto Q(s, a)$ where $Q$ is the nonstationary Q-table if $s, a$ fails child stationarity and the stationary table otherwise. The result of this is that, with respect to a given $s, a$, updates done at \cs{} iterations have no impact on updates done at non \cs{} iterations and vice versa.

To prove Theorem~4.4, we proceed by induction. As the base case, consider any terminal infostate $s$ of the game tree where there are no actions. In this case,
 all policies are no-regret. 
For the inductive hypothesis, consider any infostate $s$ and assume that \alg{} guarantees no-regret for any $s^\prime, a^\prime$ where $s^\prime$ is a descendent of $s$ and $a^\prime$ is a valid action at $s^\prime$. To show  is 
that \alg{} also guarantees the no-regret property at $s, a$.
Formally, define $\eta_{-i}^{\pi}(s) = \sum_{h \in H(s)} \eta_{-i}^{\pi}(h)$. We  write the {\em  local regret} for a policy $\pi$ as $r^\pi(s, a) = \max\br{0, \eta_{-i}^{\pi}(s) \br{ Q^{\pi}(s, a) - \E{a' \sim \pi(s)}{Q^\pi(s, a')}}}$.

Consider the sequence $\br{d_n}_n$, where $d_n \in \{0, 1\}$ represents whether the oracle reports whether $s, a$ is \cs{} after $n$ iterations  (or episodes)
of training. 
There are three  cases to consider: $\br{d_n}_n$  converges to $0$ and fails to satisfy \cs{} in the limit, converges to $1$ and satisfies \cs{} in the limit, or fails to converge (e.g., perhaps it oscillates between \cs{} and not \cs{}).

\paragraph{Case 1:} $s, a$ fails child stationarity in the limit (i.e., $\br{d_n}_n \rightarrow 0$).  

In Appendix~\ref{sec:maxcfr}, we show that \alg{} runs MAX-CFR at every $s, a$ that does not satisfy \cs{} according to the detector. Lemma~\ref{maxcfrboundhigh} shows that running MAX-CFR at a given state and action $s, a$ achieves vanishing local regret at $s, a$, and the assumption of the perfect oracle means that \alg{} will asymptotically converge to MAXCFR at this particular $s, a$. Combined with our inductive hypothesis, this  guarantees that \alg{} is no-regret for $(s, a)$ and all its descendent infostate, action pairs. %

\paragraph{Case 2:} $s, a$ satisfies child stationarity in the limit  (i.e., $\br{d_n}_n \rightarrow 1$). Let $\overline{\pi}^N$ denote the average joint policy at episode $N$.
The Multiplicative Weights / Hedge algorithm \cite{freund1997decision, arora2005fast} guarantees that selecting $\pi_{n+1}(s) \propto \text{softmax}\br{Q^{\overline{\pi}_N}(s, *)}$ %
 has $\lim_{N\rightarrow \infty} r^{\overline{\pi}_N}(s, a) = 0$. When the detector successfully detects $s, a$ as child stationary,  Algorithm~\ref{alg:mainalg} performs exactly such an update at every infostate visited (except that it uses estimated $\hat{Q}^{\overline{\pi}_N(s, a)}$ values).
We also show that the estimated $\hat{Q}^{\overline{\pi}_N}(s, a)$ values converge to the true ${Q}^{\overline{\pi}_N}(s, a)$ values, resulting in Hedge / Multiplicative weights performing the softmax update on  the 
correct counterfactual regrets. We do so as follows.
\begin{lemma}
Define $G^{*}(s, a) = \lim_{N\rightarrow \infty} G^{\overline{\pi}^N}(s, a)$. If $\br{d_t}_t \rightarrow 1$, then $G^{*}(s, a)$ is guaranteed to exist.
\end{lemma}
\begin{proof}
While $G^{*}(s, a)$ is not a well defined subgame in general, we show that it exists if $s, a$ is asymptotically \cs{}. Considering a vectorized form of the payoff entires in a finite action, normal form game, the set of all games can be
written  as a subset of $\mathcal{R}^q$, for a suitable $q$, and  forms a compact metric space.
 $\br{d_n}_n \rightarrow 1$  implies $\Delta_N \rightarrow 0$ by construction. %
Let $\pi_A^N, \pi_B^N$ represent the average joint policies as defined in Section~\ref{sec:measuringchildstationarity} at iteration $N$. 
$\overline{\pi}^N$ is, by definition, the average joint policy between $\pi_A^N$ and $\pi_B^N$. Using the symmetry of the total variation distance along with the triangle inequality, we have
\begin{align*}
& d\br{G^{\overline{\pi}^N}(s, a), G^{\overline{\pi}^{N+1}}(s, a)} \\
& = d\br{G^{\frac{1}{2} \pi_A^N + \pi_B^N}(s, a), G^{\frac{1}{2} \pi_A^{N+1} + \pi_B^{N+1}}(s, a)} \\
& \leq d\br{G^{\frac{1}{2} \pi_A^N + \frac{1}{2}\pi_B^N}(s, a), G^{\pi_A^{N}}(s, a)} \\ & + d\br{G^{\pi_A^{N+1}(s, a)}, G^{\frac{1}{2} \pi_A^{N+1} + \frac{1}{2}\pi_B^{N+1}}(s, a)} \\
& + d\br{G^{\pi_A^N}(s, a), G^{\pi_A^{N+1}}(s, a)}\\
& \leq \Delta_N + \Delta_{N+1} + \frac{1}{N+1}
\end{align*}
where we have that $d\br{G^{\frac{1}{2} \pi_A^N + \frac{1}{2}\pi_B^N}(s, a), G^{\pi_A^{N}}} < d\br{G^{\pi_B^{N}}(s, a), G^{\pi_A^{N}(s, a)}}$ due to the convexity of the total variation distance \cite{billingsley2012probability}. Additionally,  $d\br{G^{\pi_A^N}(s, a), G^{\pi_A^{N+1}}(s, a)} \leq \frac{1}{N+1}$ because, by construction,  $\pi_A^N$ and $\pi_A^{N+1}$ are identical in all but at most $\frac{1}{N+1}$ fraction of the joint policies in the support. Since $\lim_{N\rightarrow \infty} d\br{G^{\overline{\pi}^N}(s, a), G^{\overline{\pi}^{N+1}}(s, a)} = 0$, the sequence of games given by $G^{\overline{\pi}^N}(s, a)$ forms a Cauchy sequence and $G^{*}(s, a)$ is  guaranteed to exist since Cauchy sequences   converge in compact metric spaces.
\end{proof}

By the assumption that $G$ is a two-player zero-sum game, $G^{*}(s, a)$ must also be a two-player zero-sum game.
By our inductive hypothesis, \alg{} has no regret for any $s^\prime, a^\prime$ that
 is a descendent of $s, a$. This, combined with the existence of $G^{*}(s, a)$ and \citep[Theorem 2]{zinkevich2008regret} means that $\lim_{N\rightarrow \infty} \overline{\pi}^N$ converges to a Nash equilibrium of $G^{*}(s, a)$. 
 
 The minimax theorem \cite{vonneumann1928zur} states that in two-player zero-sum games, players receive the same payoff, known as the minimax value of the game, \emph{regardless of which Nash equilibrium} is discovered. Call this value $Q^*$. Furthermore, by the assumption that $\lim \br{d_t}_t \rightarrow 1$, the immediate rewards received after taking action $a$ at state $s$ must converge to some value $r^*$. As such, the estimated Q values must  converge to the correct value $\hat{Q}^{\overline{\pi}_N}(s, a) = r^* + \gamma Q^*$, as desired.
 
\paragraph{Case 3:} Imagine that $s, a$ is neither stationary nor nonstationary in the limit.
Let 
\[
A^N_{nonstat} = \{d_t \mid d_t = 0\}, \quad A^N_{stat} = \{d_t \mid d_t = 1\}.
\] 

Each of $\lim_{N\rightarrow \infty} \norm{A^N_{stat}}$ and $\lim_{N\rightarrow \infty} \norm{A^N_{nonstat}}$ must both diverge, since
 if one of them were only finitely large, then this would contradict the assumption that $(d_t)_t$ does not converge.

As described in our setup, because we track separate Q values for stationary and nonstationary updates and use a perfect  oracle, the updates done in iterations $A^N_{stat}$ have no impact on the updates done on $A^N_{nonstat}$ and vice versa. As such, consider the subset of timesteps given by $A^N_{nonstat}$. By direct reduction to Case 1, the distribution over joint policies given by sampling uniformly over $\{\pi_k \mid k \in A^N_{nonstat}\}$ 
must satisfy the no-regret property for $s, a$. 
Analogously,  the distribution over joint policies given by sampling uniformly over $\{\pi_k \mid k \in A^N_{stat}\}$  must satisfy the no-regret property for $s, a$
by direct reduction to Case 2.

As the inductive hypothesis is satisfied for both partitions $A^N_{stat}$ and $A^N_{nonstat}$,  we can continue to induct independently for each partition.
The induction terminates at the initial infostate $s_0$ of our game because we assume finite depth.  Suppose  we are left 
at this step with $K$ nonoverlapping partitions of $\{1 \cdots N\}$. Denote them $A^N_1 \cdots A^N_K$, and analogously define $\pi_N^k = \frac{1}{\norm{A^N_k}} \sum_{n \in A^N_k} \pi_n$. %
 By our inductive hypothesis and the minimax theorem, we are guaranteed that $\lim_{N\rightarrow \infty} r^{\pi_N^k}(s, a) = 0$, for all $s, a, k$. 
Since $\overline{\pi}^N = \frac{1}{N} \sum_k \norm{A^N_k} \pi_N^k$, we can write
\begin{align*}
\max_a r^{\overline{\pi}^N}(s, a) & = \max_a \frac{1}{N} \sum_k \norm{A^N_k} r^{\pi_N^k}(s, a) \\
& \leq \frac{1}{N} \sum_k \norm{A^N_k} \max_a r^{\pi_N^k}(s, a).
\end{align*}

Lastly, since $\lim_{N\rightarrow \infty} r^{\pi_N^k}(s, a) = 0$ for all $s, a, k$, it must also be true that $r^{\overline{\pi}^N}(s, a) \rightarrow 0$ as desired.

\end{proof}

{\bf Theorem 4.5.} {\em Given an appropriate choice of significance levels $\alpha$, Algorithm~\ref{alg:mainalg} asymptotically converges to the optimal policy with only a worst-case $O(A)$ factor slowdown compared to BQL in an MDP.}
\begin{proof}
We choose the significance levels $\alpha$ across different infostates as follows. As described in Algorithm~\ref{alg:mainalg}, at each infostate $s$ encountered, \alg{} samples a single action $a_{traj}$ and branches this action, regardless of whether $s, a_{traj}$ is determined to be child stationary or not.

Define the set of infostates $s_0, \dots, s_t$, where $s_0$ is the initial infostate and $s_{k}$ is the infostate observed after playing the corresponding $a_{traj}$ action at $s_{k-1}$. In other words, this is the trajectory of infostates visited from the beginning to the end of the game, taking the actions sampled by \alg{}, independent of the child stationarity detection. This ``trajectory'' is the set of infostates that \alg{} would have updated even if the environment had been fully child stationary and the direct analogue of the trajectory that BQL would have updated. Note that the chosen ``trajectory'' may change during every iteration of the learning procedure.

We select the significance levels $\alpha$ as follows. For each infostate $s_i$ on the constructed trajectory described above, set $\alpha \leq \frac{1}{A^d}$, where $A$ is the maximum number of actions at any infostate and $d$ is the depth of infostate $s_i$. At all other infostates reached during the current learning iteration, set $\alpha \leq \frac{1}{A}$.

This choice of significance levels guarantees only an $O(A)$ factor slowdown compared to BQL. 
Let $X_{\alg}(d)$ count the number of infostates of depth $d$ that \alg{} updates at each iteration.

Since the environment is an MDP, an infostate is updated only if the detection algorithm hit a false positive on every ancestor of the infostate that was not on the trajectory path.
We  bound this probability as follows.
Consider an infostate $s$ with depth $d$.
Let $d^*$ be the highest depth infostate in the ancestry of $s$ that
 was on the ``trajectory'' path. 
As false positives asymptotically occur with probability $\alpha$ in a Chi-Squared test \cite{pearson1900criterion}, the probability that $s$ is visited approaches
\[
\frac{1}{A^{d^*}} \cdot \prod_{i=d^*+1}^d \frac{1}{A} = A^{-d}.
\]

There are only at most $A^d$ nodes at depth $d$ in the tree, and thus, in expectation, we can expect at most $A^{-d} A^d = O(1)$ extra updates at level $d$ in the tree, or $O(D)$ total extra updates for a game tree of depth $D$. 
Since MAXCFR is a contraction operator over the potential function which bounds the regret (Appendix~\ref{sec:maxcfr}), additional, unnecessary updates at infostates falsely deemed to be nonstationary by the detector in Algorithm~\ref{alg:detector} will not hurt the asymptotic convergence rate, as MAXCFR also converges to the optimal policy in an MDP.

As our environment is acylic, \alg{} will visit $O(D)$ infostates in a game tree of depth $D$, the same as BQL. At each infostate, \alg{} updates every action whereas BQL only updates a single action, giving the additional $O(A)$ inefficiency relative to BQL.
\end{proof}

\section{Experimental Details}
\label{sec:hyperparameters}
\subsection{Testing Environments}

We present more detailed descriptions of our testing environments, including any modifications made to their standard implementations, below.

\subsubsection{Cartpole}

Cartpole is a classic control game from the OpenAI Gym, in which the agent must work to keep the pole atop the cart upright while keeping the cart within the boundaries of the game. While Cartpole has a continuous state space, represented by a tuple containing (cart position, cart velocity, pole angle, pole angular velocity), we discretize the state space for the tabular setting as follows:

\begin{itemize}
    \item \textbf{Cart position:} $10$ equally-spaced bins in the range $[-2.4, 2.4]$
    \item \textbf{Cart velocity:} $10$ equally-spaced bins in the range $[-3.0, 3.0]$
    \item \textbf{Pole angle:} $10$ equally-spaced bins in the range $[-0.5, 0.5]$
    \item \textbf{Pole angular velocity:} $10$ equally-spaced bins in the range $[-2.0, 2.0]$
\end{itemize}

Additionally, to make convergence feasible, we parameterize the game such that the infostates exactly correspond to the discretized states, rather than the full sequence of states seen and actions taken. Note that this means that our version of Cartpole does not satisfy perfect recall; to account for this and ensure that we are not repeatedly branching the same infostate, we prevent \alg{} and MAX-CFR from performing a CFR exploratory update on the same infoset twice in a given iteration.

The standard implementation of Cartpole is technically non-Markovian as each episode is limited to a finite number of steps ($200$ in the v0 implementation). To make the environment Markovian, we allow the episode to run infinitely, but we introduce a $1/200$ probability of terminating at any given round, thus ensuring that the maximum expected length of an episode is precisely $200$.

\subsubsection{Weighted Rock-Paper-Scissors}

Weighted rock-paper-scissors is a slight modification on classic rock-paper-scissors, where winning with Rock yields a reward of $2$ whereas winning with any other move only yields a reward of $1$ (a draw results in a reward of $0$). The game is played sequentially, but player 2 does not have knowledge of the move player 1 makes beforehand (possible via use of imperfect-information / infostates).

\subsubsection{Kuhn Poker}

Kuhn Poker is a simplified version of poker with only three cards – a Jack, Queen, and King. At the start of the round, both players receive a card (with no duplicates allowed). Each player antes $1$ – player $1$ then has the choice to bet or check an amount of $1$. If player $1$ bets, player $2$ can either call (with both players then revealing their cards) or fold. If player $1$ checks, player $2$ can either raise an amount of $1$ or check (ending the game). In the event that player $2$ raises, player $1$ then has the choice to call or fold, with the game terminating in either case.

We use the standard OpenSpiel implementation of Kuhn poker for all our experiments.

\subsubsection{Leduc Poker}

Leduc poker is another simplified variant of poker, though it has a deeper game tree than Kuhn poker. As with Leduc, there are three cards, though there are now two suits. Each player receives a single card, and there is an ante of $1$ along with two betting rounds. Players can call, check, and raise, with a maximum of two raises per round. Raises in the first round are two chips while they are four chips in the second round. 

Again, we use the standard OpenSpiel implementation of Leduc poker for all our experiments.

\subsubsection{Cartpole + Leduc Poker}

For our stacked environment, player $1$ first plays a round of Cartpole. Upon termination of this round of Cartpole, they then play a single round of Leduc poker against player $2$. Note that player $2$ only ever interacts with the Leduc poker environment.

We use the same implementations of Cartpole and Leduc poker discussed above, with the sole change being that we set Cartpole's termination probability to $1/100$ rather than $1/200$ to feasibly run MAX-CFR on the stacked environment. 

\subsection{Hyperparameters Used}

For all of the experiments in the main section of the paper, we use the same hyperparameters for all of the algorithms tested. We enumerate the hyperparameters used for BQL, MAX-CFR, and the MCCFR methods in Table \ref{fig:benchmark_hyperparameters}.

We tuned the hyperparameters for MCCFR methods, BQL, and \alg{} as to optimize performance over all five environments for which experiments were conducted.

\begin{table}[htbp]
\def\arraystretch{1.5}
\caption{Benchmark Hyperparameter Values}
\centering
\begin{tabular}{ll}
\toprule
\textup{Hyperparameters}       & Final Value\\
\midrule
Discount Factor & $\gamma = 1$ \\
BQL Temperature Schedule & $\tau_n = 10 \cdot \left(0.99\right)^{\lfloor n/50 \rfloor}$ \\
MAX-CFR Temperature Schedule & $\tau_n = 1$\\
OS-MCCFR $\epsilon$-greedy Policy & $\epsilon = 0.6$\\
\bottomrule
\end{tabular}
\label{fig:benchmark_hyperparameters}
\end{table}

For ABCs, we use the following set of hyperparameters, enumerated in Table \ref{fig:abcs_hyperparameters}. As noted in the main paper, we make a number of practical modifications to Algorithm \ref{alg:mainalg}. In particular, we use a fixed p-value threshold, allow for different temperature schedules for stationary/nonstationary infosets, choose not to use $\epsilon$-exploration, and only perform the stationarity check with some relatively small probability.

\begin{table}[htbp]
\def\arraystretch{1.5}
\caption{ABCs Hyperparameter Values}
\centering
\begin{tabular}{ll}
\toprule
\textup{Hyperparameters}       & Final Value\\
\midrule
Discount Factor & $\gamma = 1$ \\
Nonstationary Temperature Schedule & $\tau_n = 1$\\
Stationary Temperature Schedule & $\tau_n = \left(0.99\right)^{\lfloor n/20 \rfloor}$\\
Cutoff p-value & $\alpha_s = 0.05$\\
$\epsilon$-exploration & $\epsilon = 0$\\
Probability of Stationarity Check & $p_{check} = 0.05$\\
\bottomrule
\end{tabular}
\label{fig:abcs_hyperparameters}
\end{table}

\newpage
\subsection{Evaluation Metrics}

For all our experiments, we plot ``nodes touched'' on the x-axis, a standard proxy for wall clock time that is independent of any hardware or implementation details. Specifically, it counts the number of nodes traversed in the game tree throughout the learning procedure. On each environment, we run all algorithms for a fixed number of nodes touched, as seen in the corresponding figures.

For all multi-agent settings, we plot exploitability on the y-axis. For Cartpole, we plot regret, which is simply the difference between the reward obtained by the current policy and that obtained by the optimal policy ($200$ minus the current reward).

\subsection{Random Seeds}

We use three sequential random seeds $\{0, 1, 2\}$ for each experiment.

\section{Additional Experiments}
\label{sec:add_experiments}

\subsection{Stationarity Detection}

We plot percentage of infostates detected as nonstationary over time for each of the non-stacked environments in Figure \ref{fig:stat_detection}. We normalize the x-axis to plot the fraction of total runtime/nodes touched so that all environments can be more easily compared.

\begin{figure*}[!h]
    \centering
    \begin{subfigure}[t]{0.45\textwidth}
        \centering
        \includegraphics[width=\linewidth]{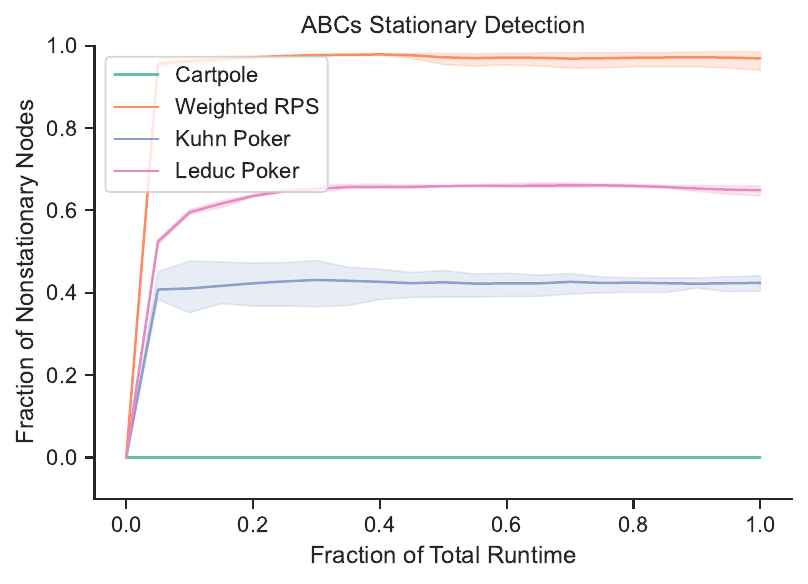} 
        \caption{Stationarity detection for Cartpole, RPS, Leduc poker, and Kuhn poker. The fraction of infostates detected as non-stationary is plotted over time.}
        \label{fig:stat_detection}
    \end{subfigure}
    \begin{subfigure}[t]{0.45\textwidth}
        \centering
        \includegraphics[width=\linewidth]{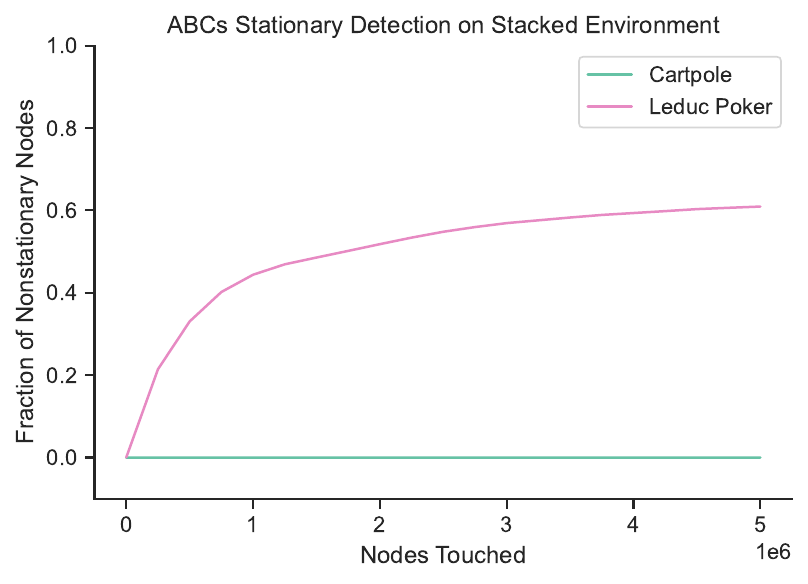} 
        \caption{Stationarity detection on stacked environment.}
        \label{fig:stacked_stat}
    \end{subfigure}
\end{figure*}

Note that the multi-agent environments possess varying levels of stationarity. The policies in weighted RPS are cyclical, meaning that the corresponding transition functions change throughout the learning procedure. In contrast, in our two poker environments, there are several transition functions that may stay fixed, either because the outcome of that round is determined by the cards drawn or because agents' policies become fixed over time. 

To take an example, in the Nash equilibrium of Kuhn poker, player 1 only ever bets in the opening round with a Jack or a King. Thus, if player 2 is in the position to call with a King, betting will necessarily yield the history in which players have cards (K, J) and player 2 calls player 1's initial bet.

To better visualize the performance of ABCs on our stacked Cartpole/Leduc environment, we plot the fraction of infostates detected as nonstationary in each environment separately in Figure \ref{fig:stacked_stat}.

\subsection{TicTacToe}

We also run experiments on TicTacToe. While TicTacToe is not a fully stationary game – as it is still a multi-agent setting – it is a perfect information game, meaning that BQL should still be able to learn the optimal policy of the game. As with Cartpole, we modify the infostates to depend only on the current state of the board. 

We plot results for ABCs and benchmarks in Figure \ref{fig:ttt}, with exploitability clipped at $10^{-5}$ to facilitate easier comparison. To accelerate convergence, we evaluate the current greedy policy rather than the average policy – note that last-iterate convergence implies average policy convergence. We additionally modify the temperature schedules for ABCs and BQL to accommodate the larger game tree. These are listed in Table \ref{fig:ttt_hyperparameters}.

\begin{table}[t]
\def\arraystretch{1.5}
\caption{TicTacToe Hyperparameter Values}
\centering
\begin{tabular}{ll}
\toprule
\textup{Hyperparameters}       & Final Value\\
\midrule
ABCs Non-\\stationary Temperature Schedule & $\tau_n = 1$\\
ABCs \\ Stationary Temperature Schedule & $\tau_n = 10 \cdot \left(0.99\right)^{\lfloor n/50 \rfloor}$\\
BQL Temperature\\ Schedule & $\tau_n = 10 \cdot \left(0.99\right)^{\lfloor n/100 \rfloor}$\\
\bottomrule
\end{tabular}
\label{fig:ttt_hyperparameters}
\end{table}

\begin{figure*}[!h]
    \centering
    \begin{subfigure}[t]{0.45\textwidth}
        \centering
        \includegraphics[width=\linewidth]{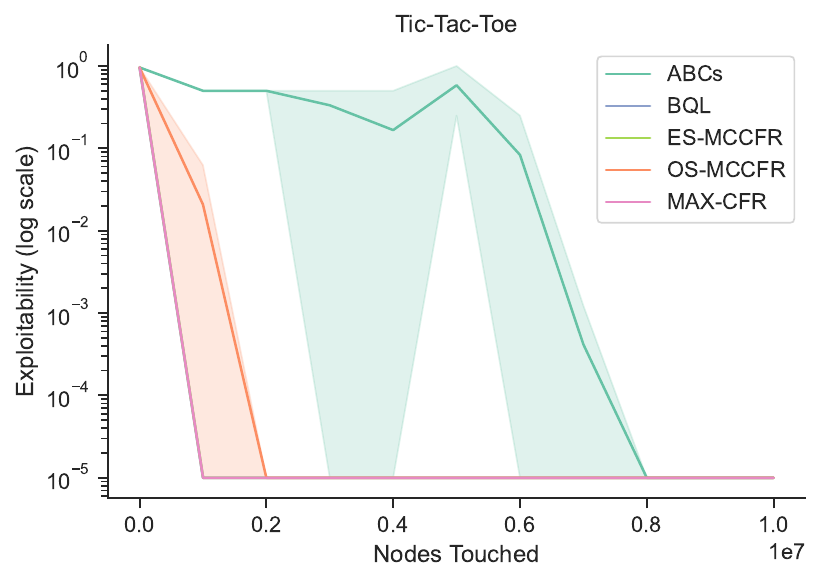}
        \caption{Performance on TicTacToe, exploitability is clipped at $10^{-5}$ and plotted on log scale.}
        \label{fig:ttt}
    \end{subfigure}
    \begin{subfigure}[t]{0.45\textwidth}
        \centering
        \includegraphics[width=\linewidth]{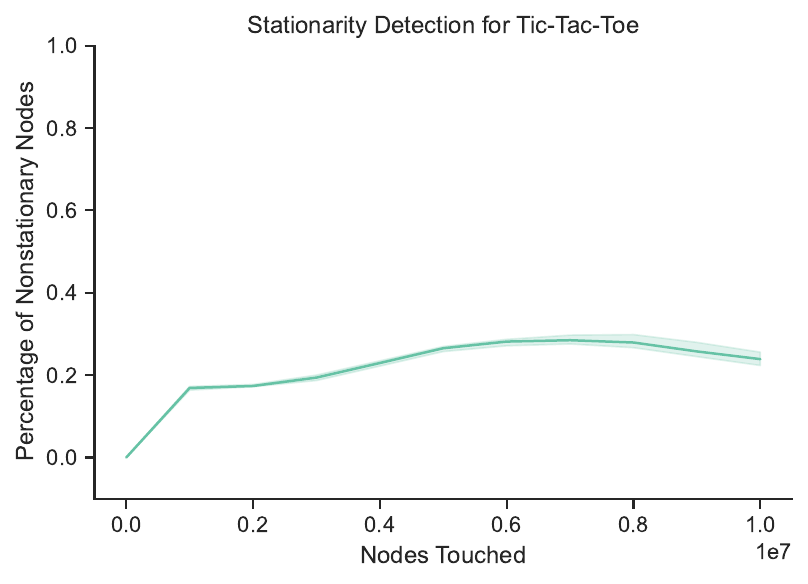}
        \caption{Nonstationary detection on TicTacToe.}
        \label{fig:ttt_stat}
    \end{subfigure}
\end{figure*}

\input{paper/additional_related_work.tex}

%% file: paper/maxcfr.tex
\section{The \cfr{} Algorithm}
\label{sec:maxcfr}
MAX-CFR is a variant of ES-MCCFR that \alg{} reverts to if the environment fails child stationarity at every possible $s, a$ (i.e., is ``maximally nonstationary").
  \cfr{} can also be viewed as a bootstrapped variant of {\em external-sampling MCCFR}~\cite{lanctot2009monte}. See Algorithm~\ref{alg:maxcfr}. 

\begin{algorithm}[h]
   \caption{\cfr}
   \label{alg:maxcfr}
\begin{algorithmic}
    \Require{$Q$, CNT initialized to $0$, discount factor $\gamma$, total episodes $N$. Initial observation and hidden states $s_0, h_0$.}
    \Statex
    \For{$n \gets 1 \textrm{ to } N$}
        \State MAXCFR$(h_0, s_0, Q, \gamma, n)$
    \EndFor
    \Statex
    \Function{MAXCFR}{$h, s, Q, \gamma, n$}
        \If{h is terminal}
            \Return{0}
        \EndIf
        \Let{$\pi^n(s)$}{$\texttt{softmax}\br{Q\br{s, *}, \tau=\frac{1}{\text{CNT}(s)}}$}
        \Let{$\text{CNT}(s)$}{$\text{CNT}(s) + 1$}

        \Statex
        \For{$a \in A(s)$}
            \Let{$\alpha$}{$1 / \text{CNT}(s)$}
            \Let{$h^\prime, s^\prime,  a^\prime, \nabla_{s, a}$}{GetChild$\br{h, s, a, n, Q, \gamma}$}
            \Let{$\nabla_{s^\prime, a^\prime}$}{$f\br{h^\prime, s^\prime, Q, \gamma, n}$}
            \Statex
            \Let{$\nabla_{s, a}$}{$\nabla_{s, a} + \text{CNT}(s^\prime) \cdot \nabla_{s^\prime, a^\prime}$}
            \Let{$Q(s, a)$}{$Q(s, a) + \alpha \nabla_{s, a}$}

            \Statex
        \EndFor
        \Let{$a^*$}{$\arg\max_a Q(s, a)$}
        \State \Return{$\nabla_{s, a^*}$}
    \EndFunction
\end{algorithmic}
\end{algorithm}

\subsection{Connecting \alg{}, MAX-CFR, and ES-MCCFR}
\label{sec:connect_cfr}

To make explicit the connection between \cfr{} and traditional external-sampling MCCFR \cite{lanctot2009monte} 
we adopt the Multiplicative Weights (MW) algorithm (also known as Hedge) as the underlying regret minimizer. In MW, a player's policy at iteration $n+1$ is given by
\[ \pi^{n+1}(s, a) = \frac{e^{\tau_n R^n(s, a)}}{\sum_{a' \in A(s)} e^{\tau_n R^n(s, a')}},\]
where $R^n(s, a)$ denotes the cumulative regret at iteration $n$, for action $a$ at infostate $s$. Hedge also includes a tunable hyperparameter $\tau_n$, which may be adjusted from iteration to iteration. Letting $\tau_n = 1$ across all iterations, we recover the softmax operator on cumulative regrets as a special case of Hedge.

We will assume there are no non-terminal rewards $r_t$, and use a discount factor of $\gamma = 1$, both of which are standard for the two-player zero-sum environments on which CFR is typically applied. We adopt $\tau_n = 1$ for simplicity, but note that the following equivalences hold for any schedule of $\tau_n$.

We first present \bootcfr{}, a ``bootstrapped" version of external-sampling MCCFR that is identical to the original ES-MCCFR algorithm.
 To make the comparison between \bootcfr{} and \cfr{} clear, we highlight the  difference between the two algorithms in cyan.
\begin{algorithm}[h]
   \caption{\bootcfr}
   \label{alg:bootcfr}
\begin{algorithmic}
    \Require{Same preconditions as MAX-CFR (Algorithm~\ref{alg:maxcfr}).}
    \Statex
    \For{$n \gets 1 \textrm{ to } N$}
        \State BOOTCFR$(h_0, s_0, Q, \gamma, n, \text{BOOTCFR})$
    \EndFor
    \Statex
    \Function{ModGetChild}{$h, s, a, \pi, Q, \gamma$}
        \State Sample $h^\prime, s^\prime \sim T^{\pi}\br{h^\prime \mid h, a}, S(h^\prime)$
        \State Sample $r \sim R^{\pi}\br{r \mid s, a, s^\prime}$
        \State \Return{$h', s', r$}
    \EndFunction
    \Statex
    \Function{BOOTCFR}{$h, s, Q, \gamma, n$}
        \If{h is terminal}
            \For{$a \in A(s)$}
                \Let{$\Delta_{s, a}$}{0}
            \EndFor
            \Return{}
        \EndIf
        \Statex
        \Let{$\pi^n(s)$}{$\texttt{softmax}\br{Q\br{s, *}, \tau=\frac{1}{\text{CNT}(s)}}$}
        \Let{$\text{CNT}(s)$}{$\text{CNT}(s) + 1$}
        \Statex
        \For{$a \in A(s)$}
            \Let{$\Delta_{s, a}$}{0}
            \Let{$h^\prime, s^\prime,  r$}{ModGetChild$\br{h, s, a, \pi^n, Q, \gamma}$}
            \For{$a' \in A(s^\prime)$}
                \Let{$\nabla_{s, a, a'}$}{$r + \gamma Q(s', a')) - Q(s, a)$}
            \EndFor
            \Statex
            \Let{$Q_{\mathrm{old}}$}{$Q(s', *)$}
            \State BOOTCFR$\br{h^\prime, s^\prime, Q, \gamma, n}$
            \Let{$Q_{\mathrm{new}}$}{$Q(s', *)$}
            \Statex
            \textcolor{cyan}{\Let{$\pi_{\mathrm{eval}}^n(s')$}{$\texttt{softmax}\br{Q_{\mathrm{old}}\br{s', *}, \tau=\frac{1}{\text{CNT}(s') - 1}}$}}
            \For{$a^\prime \in A(s^\prime)$}
                \Let{$\Delta_{s, a}$}{$\Delta_{s, a} + \pi_{\mathrm{eval}}^n(s', a') \cdot \left[\frac{1}{\text{CNT}(s)}\left(\nabla_{s, a, a'} + \text{CNT}(s^\prime) \cdot \Delta_{s^\prime, a^\prime}\right)\right]$}
            \EndFor
            \Let{$Q(s, a)$}{$Q(s, a) + \Delta_{s, a}$}
        \EndFor
        \Statex
        \Return{}
    \EndFunction
\end{algorithmic}
\end{algorithm}

\begin{lemma}
    \label{hedgerewards}
    Multiplicative Weights / Hedge is invariant to the choice of cumulative regrets or cumulative rewards/utility.
\end{lemma}
\begin{proof}
    It is well-known and easy to verify that Multiplicative Weights / Hedge is shift-invariant. That is,
    \[\mathrm{Hedge}(x)  = \mathrm{Hedge}(x + c).\]
    
    Over $N$ iterations, the cumulative counterfactual regret of not taking action $a$ at infostate $s$ is given by
    \[R^N(s, a) = \frac{1}{N} \sum_{n=1}^N v_n^{\pi_{(s\to a)}^n}(s) - v_n^{\pi^n}(s),\]
    where $\pi^n_{s\to a}$ is identical to $\pi^n$, except that action $a$ is always taken given infostate $s$. Since the second term does not depend on the action $a$, this is equivalent to using Hedge on the cumulative counterfactual rewards
    \[\frac{1}{N} \sum_{n=1}^N v_n^{\pi_{(s\to a)}^n}(s),\]
     by the shift-invariance of Hedge.
\end{proof}

\begin{lemma}
    \label{equivalentexploration}
    \bootcfr{} and external-sampling MCCFR use the same current policy and traverse the game tree in the same way.
\end{lemma}
\begin{proof}
    First, note that game-tree traversal is identical to external-sampling MCCFR by construction. For a given player, we branch all their actions and sample all opponent actions from the current policy. 
    Additionally, while external-sampling MCCFR uses softmax on cumulative regrets, by Lemma \ref{hedgerewards}, this is equivalent to softmax on cumulative rewards. Note that 
    the policy in \bootcfr{} is given by
    \begin{align*}
        \pi^n(s) &= \texttt{softmax}\left(Q(s, *), \tau = \frac{1}{\text{CNT}(s)}\right)\\
        &= \texttt{softmax}\left(Q(s, *) \cdot \text{CNT}(s)\right).
    \end{align*}
    
    As $Q(s, a)$ is the average reward from taking action $a$ at infostate $s$, it follows that $Q(s, a) \cdot \text{CNT}(s, a)$ is the cumulative reward. Further, since we branch all actions anytime we visit an infoset, then $\text{CNT}(s, a) = \text{CNT}(s)$ and the  policy in \bootcfr{} is equivalent to a softmax on cumulative rewards.
\end{proof}

\begin{theorem}
    \label{bootcfridentical}
\bootcfr{} and external-sampling MCCFR are identical.
\end{theorem}
\begin{proof}
   \bootcfr{} and external-sampling MCCFR are identical regarding traversal of the game tree and computation of the current policy (Lemma \ref{equivalentexploration}). 
Thus,  it suffices to show that the updates to cumulative/average rewards are identical.

    Let $W$ denote the set of terminal states that are reached on a given iteration of \bootcfr{} or external-sampling MCCFR. For any $z \in W$, we define $u_i(z)$ to be the reward for player $i$ at terminal history $z$.
    Consider current history $h$, current infostate $s$, and action $a$. Let $h'$ be the history reached on the current iteration after taking action $a$. For external-sampling MCCFR, the sampled counterfactual regret of taking action $a$ at infostate $s$ is given by
    \[\tilde{r}(s, a) = \sum_{z \in W} u_i(z)(\eta_i^\pi(h', z) - \eta_i^\pi(h, z)).\]
   
 By Lemma~\ref{hedgerewards}, when using Hedge, we can consider only the sampled counterfactual rewards
    \[r(s, a) = \sum_{z \in W} u_i(z)\cdot \eta^\pi(h', z).\]

    While the sampled counterfactual reward is typically only defined for non-terminal observations, for the sake of the proof, we extend the definition to terminal observations, letting $r(s, a) = u_i(z)$, where $z$ is the terminal history corresponding to observation $h$.
    
    We will prove that $\Delta_{s, a} = \frac{r(s, a) - Q(s, a)}{\text{CNT}(s)}$ for all $s$ that we visit on a given iteration and all actions $a$.

    We proceed by induction. First, we have this equivalence 
 for all terminal nodes, since $\Delta_{s, a} = 0$ at a terminal node.
    Now, consider an arbitrary infostate $s$ and action $a$, and suppose that $\Delta_{s', a'} = \frac{r(s', a') - Q(s', a')}{\text{CNT}(s', a')}$ for all actions $a'$ and $s'$ reachable from $s$ after playing action $a$. By our inductive hypothesis, we have 
    \begin{align*}
        &\nabla_{s, a, a'} + \text{CNT}(s^\prime)\cdot\Delta_{s^\prime, a^\prime} \\
        &= Q(s', a') - Q(s, a) + \text{CNT}(s^\prime) \cdot \frac{r(s', a') - Q(s', a')}{\text{CNT}(s')}\\
        &= r(s', a') - Q(s, a).
    \end{align*}
    
    Additionally, we have 
    \begin{align*}
        r(s, a) &= \sum_{z \in W} u_i(z)\cdot \eta_i^\pi(h', z)\\
        &= \sum_{z \in W}u_i(z) \cdot \left[\sum_{a' \in A(s')} \pi(a'|s') \cdot \eta_i^{\pi}(h''_{a'}, z)\right]\\
        &= \sum_{a' \in A(s')} \pi(a'|s') \sum_{z \in W}u_i(z) \eta_i^{\pi}(h''_{a'}, z)\\
        &= \sum_{a' \in A(s')} \pi(a'|s') \cdot r(s', a'),
    \end{align*}
    where $h''_{a'}$ is the history reached after playing action $a'$ at history $h'$. It follows that
    \begin{align*}
        \Delta_{s, a} &= \sum_{a' \in A(s')} \frac{1}{CNT(s)} \cdot \pi(a'|s') \cdot \left[\nabla_{s, a, a'} + \text{CNT}(s^\prime) \cdot\Delta_{s^\prime, a^\prime} \right]\\
        &= \sum_{a' \in A(s')} \frac{1}{CNT(s)} \cdot \pi(a'|s') \cdot \left[r(s', a') - Q(s, a) \right]\\
        &= \frac{1}{\text{CNT}(s)} \cdot \left[-Q(s, a) + \sum_{a' \in A(s')} \pi(a'|s') \cdot r(s', a') \right]\\
        &= \frac{r(s, a) - Q(s, a)}{CNT(s)},
    \end{align*}
    as desired. The Q-update becomes
    \begin{align*}
        Q(s, a) &\gets Q(s, a) + \frac{r(s, a) - Q(s, a)}{CNT(s)}\\
        &= \frac{(CNT(s) - 1) \cdot Q(s, a) + r(s, a)}{CNT(s)},
    \end{align*}
    and corresponds precisely to updating the average Q-value with the external-sampling sampled counterfactual reward.
\end{proof}

\begin{corollary}
\label{cfrequivalence}
\cfr{} and External-Sampling MCCFR are identical, up to minor differences in the calculation of an action's expected utility/reward.
\end{corollary}

\begin{proof}
\arxiv{Spell it out since we change the formatting and not just a single line}
One can verify that there is a single difference between \bootcfr{} and \cfr{}, where the line highlighted in cyan in Algorithm~\ref{alg:bootcfr} is replaced by
\[\pi_{\mathrm{eval}}^n(s') \gets \texttt{hardmax}\br{Q_{\mathrm{new}}\br{s', *} \cdot \text{CNT}(s')},\]
where $\texttt{hardmax}$ places all probability mass on the entry/action with the highest value. Thus, instead of calculating the sampled reward of taking action $a$ at infostate $s$ with respect to the current softmax policy, \cfr{} calculates the sampled reward with respect to a greedy policy that always selects the action with the greatest cumulative reward/smallest cumulative regret. This greedy policy is taken with respect to the new Q-values
 after they have been updated on the current iteration.  

Importantly, the current policy and cumulative policy of all players is still calculated using softmax on cumulative rewards. It is only the calculation of sampled rewards that relies on this new hardmax policy.
\end{proof}

Finally, we can relate MAX-CFR to \alg{}. We have the following simple relationship.
\begin{theorem}
    Up to $\epsilon$-exploration, \alg{} reduces to MAX-CFR in the case where all $(s, a)$ pairs are detected as nonstationary.
\end{theorem}
\begin{proof}
    This follows  from the definition of Algorithm \ref{alg:mainalg}, where removing the $\epsilon$-exploration  recovers MAX-CFR.
\end{proof}

\begin{theorem}
\label{maxcfrisnoregret}
With high probability, MAX-CFR minimizes regrets at a rate of $O\br{\frac{1}{\sqrt{N}}}$.
\end{theorem}
\begin{proof}
Written in terms of Q-value notation, the local regrets that are minimized by CFR(Hedge) are given by
\[
r^N(s, a) = \eta_{-i}^{\pi}(s) \br{Q^{\pi}(s, a) - \E{a \sim \pi(s)}{Q^\pi(s, a)}}.
\]
Let 
\[
r^{N+}(s, a) = \max\br{0, r^N(s, a)}
\]

Define $\pi^{n}_{s, a}$ as the policy in which:
\begin{enumerate}
    \item All players except player $P(s)$ follow $\pi$ at all states in the game,
    \item Player $P(s)$ plays as necessary to reach $s$ and plays action $a$ at $s$, and 
    \item Player $P(s)$ plays the action with the highest Q value at $s'$, where $s'$ is the state directly following  $s$.
\end{enumerate} 

Policy $\pi^{n}_{s, a}$ is identical to the counterfactual target policy that CFR(Hedge) follows, except that $\pi^{n}_{s, a}$ alters the policy at the successor state $s'$ to follow the action with the highest Q value instead of sampling an action according to $\pi^n$. 
Analogously, define the modified MAXCFR local regrets as
\begin{align*}
r^N_{MAXCFR}(s, a) & = \eta_{-i}^{\pi}(s) \br{Q^{\pi^{n}_{s, a}}(s, a) - \E{a \sim \pi(s)}{Q^\pi(s, a)}} \\
r^{N+}_{MAXCFR}(s, a) & = \max\br{0, r^N_{MAXCFR}(s, a)}
\end{align*}

These regrets are identical to the standard CFR regrets $r^N(s, a)$,
 except for the fact that the target policy is modified to $\pi^{n}_{s, a}$.
\begin{lemma}
\label{lem:max}
For any $s', \pi$ and any possible $Q$ values,
\[
\max_a Q(s', a) - \E{a \sim \pi}{Q(s', a)} \geq 0.
\]
\end{lemma}
\begin{proof}
This follows immediately from the properties of the max operator.
\end{proof}
\begin{lemma}
\label{maxcfrboundlow}
$r^N(s, a) \leq r^N_{MAXCFR}(s, a)$.
\end{lemma}
\begin{proof}

By definition of $\pi^{n}_{s, a}$, we can write:
\begin{align*}
Q^{\pi^{n}_{s, a}}(s, a) &= Q^\pi(s, a) \\
& + \E{s' \sim T^\pi(s' \mid s, a)}{\max_{a'} Q(s', a')
- \E{a' \sim \pi}{Q(s', a'}}.
\end{align*}

Since the second term is strictly nonnegative by Lemma~\ref{lem:max},
we have $R^N_{MAXCFR}(s, a) > R^N(s, a)$.
\end{proof}

\begin{lemma}
\label{maxcfrboundhigh}
Let $\Delta$ be the difference between the lowest and highest possible rewards achievable for any player in $G$. Let $N_{s, a}$ denote the number of times MAX-CFR has visited $s, a$. We have
$R^N_{MAXCFR}(s, a) \leq \frac{\Delta}{\sqrt{N_{s, a}}}$.
\end{lemma}
\begin{proof}
Since we choose all policies over a softmax, all policies have full support and $\eta^\pi_{-1}(s) > 0$ for all $s$. Thus, $N_{s, a} = N$ in MAX-CFR.
Noting that MAX-CFR chooses its policy $\pi^{n}_{s, a}$ proportionally to $\exp\br{r^{N}_{MAXCFR}(s, a)}$, it follows from the Multiplicative Weights algorithm \cite{freund1997decision, arora2005fast}, that the regrets given by $r^N_{MAXCFR}(s, a)$ after $N$ iterations are bounded above by $\frac{\Delta}{\sqrt{N}}$ with high probability. 
\end{proof}
Notice that like CFR, the guarantee of Lemma~\ref{maxcfrboundhigh} holds \emph{regardless} of the strategy employed at all other $s, a$ in the game, so long as the policy at iteration $n$ chooses action $a$ at state $s$ proportionally to $\exp\br{r^{N}_{MAXCFR}(s, a)}$.
Since Lemma~\ref{maxcfrboundlow} shows that local regrets for MAXCFR form an upper bound on the standard CFR regrets (over the sequence of joint policies $\pi^1 \cdots \pi^N$), we can apply \citep[Theorem 3]{zinkevich2008regret} along with Lemma~\ref{maxcfrboundhigh} to show that the overall regret of the game is,
\begin{align*}
    r^G &\leq \norm{\mathcal{S}}\norm{\mathcal{A}} \max_{s, a} r^{N+}_{MAXCFR}(s, a)\\
    &\leq \norm{\mathcal{S}}\norm{\mathcal{A}} \frac{\Delta}{\sqrt{N}} = O\br{\frac{1}{\sqrt{N}}}
\end{align*}

\end{proof}

%% file: paper/additional_related_work.tex
\section{Additional Related Work}
We expand on the related work from Section~\ref{sec:related}.
\subsection{Methods from Game Theory}
While CFR guarantees polynomial time convergence in MDPs (assuming the MDP satisfies perfect recall), empirically it performs far slower than its corresponding reinforcement learning counterparts despite the same worst case bound. This is primarily because, in practice, reinforcement learning methods such as PPO \cite{schulman2017proximal} or DQN \cite{Mnih2015DQN} are able to learn reasonable policies in benchmarks such as the Atari Learning Environment \cite{bellemare13arcade} while exploring only a tiny fraction of the possible states in the environment. Meanwhile, CFR requires updating every infostate in the entire game tree at every iteration of the game.

\paragraph{Monte Carlo Counterfactual Regret Minimization} External sampling Monte-Carlo CFR (ES-MCCFR) and Outcome sampling Monte-Carlo CFR (OS-MCCFR) \cite{lanctot2009monte} are two methods that have been proposed to alleviate this issue. ES-MCCFR samples a single action for each player other than the player conducting the learning procedure (including the chance player) for each iteration, resulting in substantially fewer updates compared to the full version of CFR. However, because external sampling still requires performing a single update at every infostate of every player, it still requires $O(A^D)$ updates per iteration, where $A$ is the number of actions and $D$ is the maximum depth of the game tree, with respect to any single player. OS-MCCFR is similar to ES-MCCFR except that it only samples a single trajectory and then corrects its estimates of the counterfactual values using importance sampling. While this update is indeed a trajectory style update like Q-learning, the variances of the estimated values are exceedingly high as they involve dividing by the reach probability of reaching an infostate, which will be on the order of $O(A^{-D})$. As such, OS-MCCFR will still require on the order of $O(A^D)$ trajectories before it learns a reasonable policy.

Several methods have been proposed to empirically improve the asymptotic convergence rate with respect to the number of iterations. CFR+ \cite{tammelin2014solving} is a popular method that modifies the regret estimate to be an upper bound of the actual regrets. The same work also introduced alternating updates, which is only updating the regrets of each player every $P$ iterations, where $P$ is the number of players in the game. Linear CFR \cite{brown2019solving} and its follow up work greedy weights \cite{zhang2022equilibrium} modify the weighting scheme of the CFR updates. All the above methods empirically improve the convergence rate but maintain the worst case bound.
\subsection{Methods from Reinforcement Learning}
While reinforcement learning was designed for finding optimal policies on MDPs, many attempts have been made to adapt such algorithms to the multi-agent setting. Algorithms such as Q-learning and PPO are generally not guaranteed to converge to equilibria in such settings \cite{brown2019solving}.

\paragraph{AlphaStar} AlphaStar \cite{vinyals2019grandmaster} attempted to learn the real-time strategy game of Starcraft via reinforcement learning. Since Starcraft is a two-player imperfect information game, to combat the nonstationarity involved in playing multiple different opponents who required different strategies to defeat, AlphaStar trained against a league of agents instead of a single agent as is usually the case in self play. Such a league simulates the distribution of human adversaries AlphaStar is likely to encounter while playing on the Blizzard Starcraft ladder. While this method performed reasonable well in practice, it did not come with guarantees of convergence to the Nash equilibria of the game.

\paragraph{OpenAI Five} OpenAI Five \cite{openai2019five} was a similar project to learn DOTA 2, a popular multiplayer online battle arena video game. In order to combat the nonstationarity, OpenAI Five played in a restricted version of DOTA 2, in which aspects of the gave that involved imperfect information (such as wards which granted vision of enemy areas) were removed from the game. Indeed, while the resulting agent was able to beat some of the best human teams in the world online players quickly found strategies that coule exploit the agents.

\paragraph{RMAX} On the more theoretical side, RMAX \cite{brafman2002rmax} is a classic algorithm from reinforcement learning that achieves both optimal reward in an MDP and the minimax value of a stochastic game in polynomial time. RMAX operates by maintaining an optimistic estimate of the Q-values of each state which upper bounds the possible value of each state and thus encourages exploration of such states. This optimistic upper bound ensures However, there are several differences between RMAX and \alg{}. For starters, RMAX is not model free -- it requires a model of the environment to estimate the Q values. Additionally, unlike CFR (and by extension \alg), it is not Hannan consistent, meaning it will not play optimally against an arbitrary adversary in the limit. On a more practical level, RMAX leads to substantially more exploration than its related methods such as Q-learning by design, and as a result is used less often in practice compared to variants such as Q-learning which often are able to learn reasonable policies while exploring only a tiny fraction of the environment.
\paragraph{Nash Q-Learning} Nash Q-Learning \cite{hu2003nash} is an adaptation of the standard Q-learning algorithm finding the minimax value of stochastic game, which corresponds to the Nash equilibrium in two-player zero-sum games. Instead of doing a standard Q-learning update, Nash Q-learning calculates the Nash equilibrium of the subgame implied by the Q-values of all players and performs an update according to that policy, instead of the policy being currently followed by all players. However, there are several important differences between \alg{} and Nash Q-learning. Firstly, in our multi-agent learning setup, \alg{} is run for each player of the game, but the learning algorithms for each player are run separately, with no interaction between players except via the game itself. In contrast, Nash-Q requires a centralized setup in which players collectively choose their individual strategies as a function over the Q-values over \emph{all players in the game}. In this sense, Nash Q-learning is a collective algorithm for finding equilibria of the game as opposed to a learning algorithm trying to optimize the reward of each agent. Additionally, Nash Q-learning requires a Nash equilibrium solver for every update of the Q-values. While this is possible in polynomial time in a two-player zero-sum game, it is more expensive than the analogous update in \alg{} and Q-learning which is done in constant time.

\subsection{Unified Methods from Game Theory and Reinforcement Learning}

\paragraph{Magnetic Mirror Descent} Magnetic Mirror Descent (MMD) \cite{sokota2023unified} is similarly a unified algorithm capable of solving for equilibria in two-player zero-sum games and in single player settings. 
However, unlike \alg, MMD does not adaptively determine whether or not to branch infostates based on whether or not they are stationary. As such, it cannot simultaneously guarantee convergence to Nash in two-player zero-sum games and matching performance against BQL with the same set of hyperparameters.
MMD runs experiments on a variety of single and multi-agent environments, but requires different hyperparameters for each one, most importantly those governing how many infostates to explore. In contrast, \alg{} manages to roughly match the performance of Boltzmann Q-Learning in stationary environments and the performance of counterfactual regret minimization —— all with the same set of hyperparameters due to adaptively choosing how much of the tree to explore based on how stationary the infostate is.

\paragraph{Player of Games} Player of Games \cite{Schmid21PoG} is an algorithm that is able to simultaneously achieve superhuman status on many perfect and imperfect information games by combining lessons from both AlphaZero and counterfactual regret minimization. Instead of exploring the whole game tree as is the case with standard CFR, Player of Games expands a fixed number of infostates at each timestep. This is in contrast to \alg{} which performs an adaptive choice on how many children to expand depending on whether the infostate and action satisfies child stationarity.
As a result, PoG performs asymptotically slower than its counterparts designed for perfect information / single agent environments as a price for its generality. In contrast, we show that \alg{} provably has only a constant factor slowdown compared to its reinforcement learning counterpart in the limit case for a stationary environment by converging to updating the same set of infostates as BQL.